%% file: main.tex
\documentclass[nohyperref]{article} %
\usepackage[accepted]{icml2023}

\input{header}

\icmltitlerunning{FARE: Provably Fair Representation Learning with Practical Certificates}

\begin{document} 
  
\twocolumn[

\icmltitle{FARE: Provably Fair Representation Learning with Practical Certificates}

\begin{icmlauthorlist}
    \icmlauthor{Nikola Jovanovi\'c}{eth}
    \icmlauthor{Mislav Balunović}{eth}
    \icmlauthor{Dimitar I. Dimitrov}{eth}
    \icmlauthor{Martin Vechev}{eth}
\end{icmlauthorlist}

\icmlaffiliation{eth}{Department of Computer Science, ETH Zurich}
\icmlcorrespondingauthor{Nikola Jovanovi\'c}{nikola.jovanovic@inf.ethz.ch}

\icmlkeywords{Machine Learning, Fairness}

\vskip 0.3in 
]

\printAffiliationsAndNotice{}

\begin{abstract}
\input{src/abstract}
\end{abstract}

\input{src/intro}
\input{src/related}

\input{src/background}

\input{src/requirements}
\input{src/proof}
\input{src/tree}
\input{src/experiments}
\input{src/limitations}
\input{src/conclusion}

\bibliography{references}
\bibliographystyle{icml2023}

\vfill
\clearpage

\appendix 

\input{src/appendix}

\end{document}

%% file: header.tex
\input{math_commands.tex}

\usepackage[hidelinks,colorlinks]{hyperref}
\usepackage{url}

\usepackage{tikz}
\usepackage{times}
\usepackage{wrapfig}
\usepackage{booktabs}
\usepackage{multirow}
\usepackage{microtype}
\usepackage{graphicx}
\usepackage{amsmath}
\usepackage{mathtools}
\usepackage{enumitem}
\usepackage{eqparbox}
\usepackage{amsthm}
\usepackage{amssymb}
\usepackage{bbm}
\usepackage{fancybox} 
\usepackage{xcolor}
\usepackage{amsfonts}
\usepackage[capitalize,noabbrev]{cleveref}
\usepackage{natbib}
\usepackage{xspace}
\usepackage{subcaption}

\usetikzlibrary{decorations.pathreplacing,calligraphy}

\definecolor{nicerandomblue}{HTML}{1C70C3}
\hypersetup{citecolor=nicerandomblue}

\newcommand{\cert}{\ensuremath{T^\star(n, D)}}
\newcommand{\xdi}{\ensuremath{\mathcal{X}}}
\newcommand{\zdi}{\ensuremath{\mathcal{Z}}}

\newcommand{\eg}{e.g., }
\newcommand{\ie}{i.e., }

\newcommand{\indicator}[1]{\mathbbm{1}\{{#1}\}}

\newcommand{\tool}{\textsc{FARE}\xspace}
\newcommand{\toollong}{Fairness with Restricted Encoders\xspace}

\theoremstyle{plain}
\newtheorem{theorem}{Theorem}[section]

\newtheorem{lemma}[theorem]{Lemma}

\theoremstyle{definition}
\newtheorem{definition}[theorem]{Definition}

\theoremstyle{remark}

%% file: math_commands.tex
\usepackage{amsmath,amsfonts,bm}

\def\eqref#1{equation~\ref{#1}}

\def\1{\bm{1}}

\def\vx{{\bm{x}}}

\def\vz{{\bm{z}}}

\DeclareMathAlphabet{\mathsfit}{\encodingdefault}{\sfdefault}{m}{sl}
\SetMathAlphabet{\mathsfit}{bold}{\encodingdefault}{\sfdefault}{bx}{n}

\newcommand{\E}{\mathbb{E}}

%% file: src/abstract.tex
Fair representation learning (FRL) is a popular class of methods aiming to produce fair classifiers via data preprocessing. Recent regulatory directives stress the need for FRL methods that provide \emph{practical certificates}, \ie provable upper bounds on the unfairness of any downstream classifier trained on preprocessed data, which directly provides assurance in a practical scenario. Creating such FRL methods is an important challenge that remains unsolved. In this work, we address that challenge and introduce FARE (\emph{Fairness with Restricted Encoders}), the first FRL method with practical fairness certificates. FARE is based on our key insight that restricting the representation space of the encoder enables the derivation of practical guarantees, while still permitting favorable accuracy-fairness tradeoffs for suitable instantiations, such as one we propose based on fair trees. To produce a practical certificate, we develop and apply a statistical procedure that computes a finite sample high-confidence upper bound on the unfairness of any downstream classifier trained on FARE embeddings. In our comprehensive experimental evaluation, we demonstrate that FARE produces practical certificates that are tight and often even comparable with purely empirical results obtained by prior methods, which establishes the practical value of our approach.  

%% file: src/intro.tex
\section{Introduction} \label{sec:intro}

It has been repeatedly shown that machine learning systems deployed in real-world applications propagate training data biases \citep{CorbettPFGH17,KleinbergMR17}.
This is especially concerning in decision-making applications which use data representing humans (\eg financial or medical), where such biases can lead to unfavorable treatment of certain population subgroups \citep{brennan2009compas,barocas2016big}.
For instance, a loan prediction system might recommend rejection based on a \emph{sensitive attribute}, such as race or gender.
The relevance of fairness (as perceived by companies) has increased the most over the past year compared to any other potential risk of AI~\citep{Chui21,Benaich21}.

\begin{figure*}[t] 
    \centering
    \includegraphics[width=0.87\textwidth]{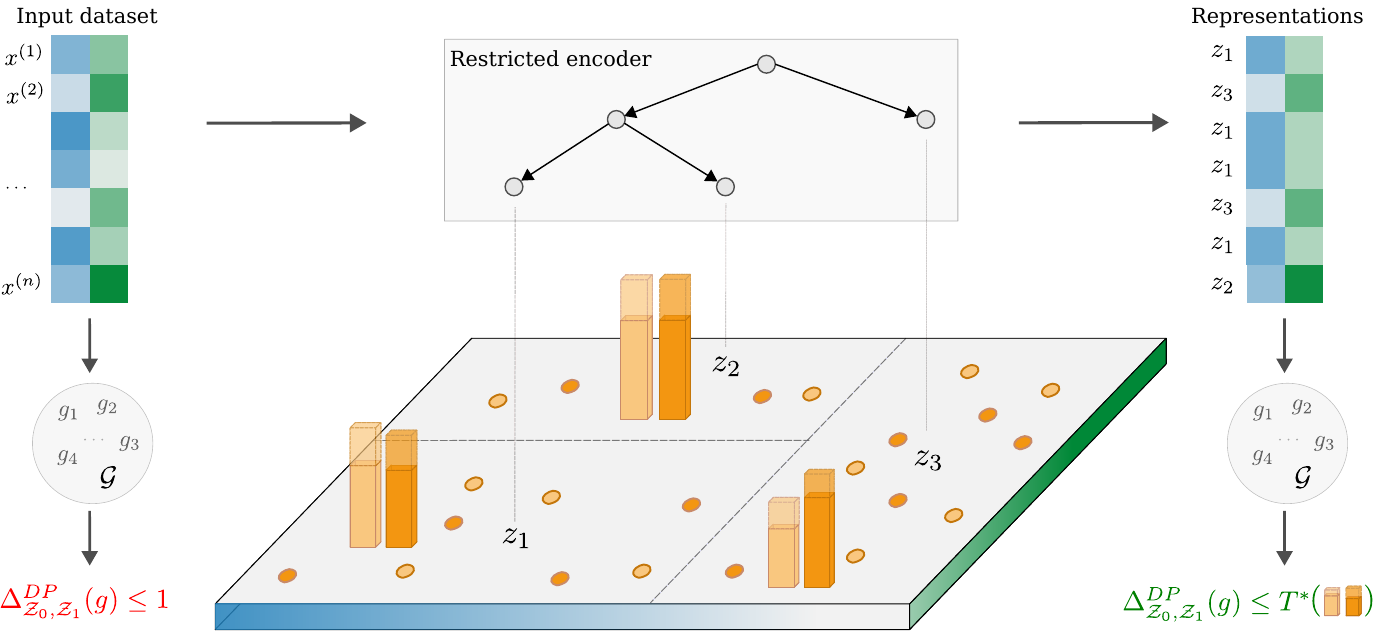}
    \caption{Overview of \tool. The input dataset is transformed into fair representations using a restricted encoder. FARE can compute a certificate $T^\star$ which provably upper-bounds the unfairness of any classifier $g \in \mathcal{G}$ trained on the representations.}
    \label{fig:accept}
\end{figure*} 
 
\paragraph{Fair representation learning}
A promising approach to addressing this issue is \emph{fair representation learning} (FRL, \citet{LFR}), which relies on {separation of responsibility} between two parties---a \emph{data producer} and a \emph{data consumer}. The key promise of FRL is that the data producer, who uses an encoder $f$ to transform each point $\vx \in \mathcal{X}$ into a debiased representation $\vz$, can ensure that these representations can be directly given to any {data consumer} aiming to solve a prediction task, such that \emph{any} classifier $g\in\mathcal{G}$ they train will have favorable fairness, even if they are unaware of fairness or actively aim to discriminate. 
 
\paragraph{The need for certificates}
Recent regulatory directives \citep{ftc2021ai,eu2021proposal} \emph{demand} that parties aiming to deploy ML models ensure the fairness of predictions \citep{DworkHPRZ12}.
In high-stakes applications (\eg judicial), failing to {guarantee} fairness may thus lead to fines or reputational harm~\citep{wired}, in addition to causing societal harm.
This makes it essential that data producers provide \emph{certificates} guaranteeing that any classifier trained on their representations must be fair, which follows the trends in other areas in which certificates are of interest, such as privacy~\citep{DPSGD} or robustness~\citep{AI2}.

\paragraph{Practical certificates}
To achieve this goal, we are interested in \emph{practical} certificates, which upper bound a fairness metric (such as demographic parity distance) \emph{with high probability}, and are computed using a \emph{finite dataset} that the data producer can access. The certificate should not make \emph{any additional assumptions} about the data distribution, and it should hold for \emph{any model} trained on the representations.
Finally, a certificate should provide tight, non-vacuous bounds that can be \emph{computed on real datasets}.
While there has been a lot of work in this direction, none of the prior approaches satisfy all of these requirements (see \cref{sec:requirements}).

\paragraph{This work}
We propose \tool (\toollong, \cref{fig:accept})---the first FRL method that provides practical certificates.
Our key insight is that using a \emph{restricted encoder}, \ie limiting possible representations to a finite set $\{\vz_1, \ldots, \vz_k\}$, allows us to derive a practical statistical procedure that computes a high-confidence upper bound on the unfairness of any $g$, \ie a fairness certificate.
We instantiate this with an encoder based on decision trees, leading to an efficient end-to-end FRL method, producing fair representations with practical fairness certificates.

Concretely, FARE starts from a given set of data samples $\{\vx^{(1)}, \ldots, \vx^{(n)}\}$ (left in \cref{fig:accept}), partitions the input space into $k$ \emph{cells} (middle plane, here $k=3$) using the decision tree encoder, and maps all samples from the same cell $i$ into the same representation $\vz_i$ (right).
As usual in FRL, training classifiers on $\vz_i$ improves fairness at the slight cost of accuracy.
However, the key advantage of FARE is that using a restricted encoder allows us to estimate the distribution of sensitive groups in each cell---namely, we empirically estimate $P(s=0 | \vz_i)$ and $P(s=1 | \vz_i)$ (solid color orange bars) for all $\vz_i$, and further upper-bound those values with high probability (transparent bars).
This in turn leads to the key feature of FARE: a certificate $T^\star$, \ie a tight upper bound on the unfairness of any $g \in \mathcal{G}$, where $\mathcal{G}$ is the set of all downstream classifiers that can be trained on $\vz_i$.

In experiments on several real datasets, we demonstrate that FARE certificates are relatively tight and thus {practical} (as defined above and elaborated on in later sections).
Further, we show that downstream classifiers trained on FARE representations can achieve empirical accuracy-fairness tradeoffs that are comparable to those of methods from prior work.

We believe our work represents a major step towards solving the important problem of preventing the deployment of discriminatory machine learning models in a provable way.

\paragraph{Main contributions} Our key contributions are:
\begin{itemize}
    \item A practical statistical procedure that, for a restricted encoder, produces a \emph{practical certificate} (see \cref{sec:requirements}), upper-bounding the unfairness of any downstream classifier trained on its representations (\cref{sec:proof}).
    \item An end-to-end FRL method, \tool, that instantiates this approach with a fair decision tree encoder (\cref{sec:tree}), applying our statistical procedure to augment the representations with a practical certificate. The implementation of \tool is publicly available at \url{https://github.com/eth-sri/fare}.
    \item An extensive experimental evaluation in several settings, demonstrating favorable empirical fairness results, as well as practical certificates, which were out of reach for prior work. Interestingly, FARE certificates are often comparable to purely empirical results of existing FRL methods (\cref{sec:experiments}).
\end{itemize}

%% file: src/related.tex
\section{Related Work} \label{sec:related}
We discuss related work on FRL for group fairness, provable fairness in orthogonal settings, and fair decision trees.

\paragraph{FRL for group fairness}
Following~\citet{LFR} which originally introduced FRL, a plethora of different methods have been proposed based on optimization~\citep{optimized-preprocess, fairpath}, adversarial training~\citep{censoring, adv-feat-learn, LAFTR, MIFR, feng2019adv, MaxEnt-ARL, fair-universal, sIPM}, variational approaches~\citep{VFAE, CVIB, FarconVAE-G, FairDisCo}, disentanglement~\citep{orthogonal-disentangled}, mutual information~\citep{FCRL, SoFair}, and normalizing flows~\citep{FNF, FairNF}. 
The key issue is that most of these methods produce representations with no fairness certificate, meaning that a model trained on their representations could have much worse fairness.
In fact, prior work~\citep{elazar2018removal, xu2020theory, FCRL, gitiaux} has shown that adversarial training methods often significantly overestimate the fairness of their representations. 
We analyze FRL works attempting to provide guarantees in detail in \cref{sec:requirements}.

\paragraph{Provable fairness in other settings}
Numerous related works on provable fairness operate in a different setting than ours---we investigate group fairness via FRL and scenarios where separation of responsibility is crucial, and following prior work, focus on other FRL methods in our evaluation.
Several FRL methods have proposed approaches for learning individually fair representations~\citep{iFair, LCIFR, LASSI}, a different notion of fairness than group fairness.
Prior work has also examined fairness guarantees in problem settings such as ranking~\citep{FairRank}, distribution shifting~\citep{kang2022CertDistFair, jin2022CertDistFair}, in-processing~\citep{feldman2015disparate, donini2018erm, celis19meta, Profitt}, post-processing~\citep{PostprocessIndFair,Fairee}, and meta-learning~\citep{FairMeta}.
 
\paragraph{Decision trees for fairness}
The line of work focusing on adapting decision trees to fairness concerns includes a wide range of methods which differ mainly in the branching criterion. Common choices include variations of Gini impurity~\citep{Kamiran, FairForests, Faht}, mixed-integer programming~\citep{OptimalFairTrees, SynthesisFairTree} or AUC~\citep{Barata}, while some apply adversarial training~\citep{AdvDecTree, FairTrainingIndFair}.
Further, some works operate in a different setting such as online learning~\citep{Faht} or post-processing~\citep{Eifffel}.
The only works in this area that offer provable fairness guarantees are~\citet{FairTrainingIndFair}, which certifies individual fairness for post-processing, ~\citet{CertBiasTrees}, which certifies that predictions will not be affected by data changes, and~\citet{Profitt}, who propose an in-processing method for training fair trees with zero-knowledge proofs. These fundamentally differ from our FRL setting where the goal is to certify fairness of any downstream classifier.

%% file: src/background.tex
\section{Background} \label{sec:background}

We set up the notation and background on FRL, and recall some key results from prior work which we build upon.

\paragraph{Fair classification} 
Assume tuples \mbox{$(\vx, s) \in \mathbb{R}^d \times \{0,1\}$} from distribution $\xdi$, where each datapoint belongs to a group with respect to a sensitive attribute $s$. 
We focus on binary classification, \ie given $y \in \{0,1\}$ for each datapoint, the goal is to build a model $g \colon \mathbb{R}^d \to \{0, 1\}$ that predicts $y$ from $\vx$. 
Besides maximizing accuracy of $g$, we aim to maximize its fairness with respect to $s$ according to some definition, often for a slight accuracy cost.
While we consider binary $s$ and $y$, our results can be easily extended to other settings (see~\cref{app:ssec:multi}).

\paragraph{Fair representation learning}
In FRL, a \emph{data producer} applies an encoder \mbox{$f \colon \mathbb{R}^d \to \mathbb{R}^{d'}$} to obtain a \emph{representation} $\vz = f(\vx)$ of each datapoint, inducing a joint distribution $\zdi$ of $(\vz, s)$. 
The representations are then published and used by various \emph{data consumers}, who train downstream classifiers $g$ to predict some $y$ from $\vz$, \ie now we have $g \colon \mathbb{R}^{d'} \to \{0,1\}$. 
The key advantage of FRL is that by ensuring fairness properties of $\vz$, we can limit the unfairness of \emph{any} $g$ trained on data from $\zdi$, for any number of consumers.
Namely, there is a separation of responsibility: data producers need to ensure fairness, such that even fairness-agnostic or adversarial consumers are unable to discriminate.

\paragraph{Fairness metrics} 
Let $\zdi_0$ and $\zdi_1$ denote conditional distributions of $\vz$ where $s=0$ and $s=1$, respectively. 
We aim to minimize the \emph{demographic parity (DP) distance} (\ie \emph{statistical parity}) of $g$:
\begin{equation}
    \Delta^{DP}_{\zdi_0, \zdi_1}(g) := \left|\E_{\vz \sim \zdi_0}[g(\vz)] - \E_{\vz \sim \zdi_1}[g(\vz)]\right|.
    \nonumber
\end{equation}  
This metric choice is motivated by prior work and enables comparison with a wide range of baselines, as many consider only DP distance~\cite{CVIB,FCRL,sIPM}---other definitions may be more suitable for particular use-cases \citep{eqodds}, and our method can be adapted to support them.
For instance, equalized odds and equal opportunity correspond to DP on $\zdi_0$ and $\zdi_1$ conditioned on the target label $y$, enabling application of our statistical procedure given in \cref{sec:proof}---see~\cref{app:ssec:metrics} for more details and experimental results.

\paragraph{Towards provable FRL}
The goal of \emph{provable FRL} is for the data producer to provide a \emph{fairness certificate} $T^\star \in \mathbb{R}$, guaranteeing that the DP distance of \emph{any} classifier trained by \emph{any} data consumer on representations from $\mathcal{Z}$ is not higher than $T^\star$.  
We now recall results from prior work, which we also utilize as a first step towards provable FRL.

Consider \mbox{$h \colon \mathbb{R}^{d'} \to \{0,1\}$}, the adversary predicting group membership $s$, aiming to maximize its balanced accuracy:
\begin{align} \label{eq:bacc}
    &{BA}_{\zdi_0, \zdi_1}(h):= \nonumber \\ & \frac{1}{2} \left( \E_{\vz \sim \mathcal{Z}_0}[1 - h(\vz)] + \E_{\vz \sim \mathcal{Z}_1}[h(\vz)] \right).
\end{align}
Let $h^\star$, such that for all $h$, $BA_{\zdi_0, \zdi_1}(h^\star) \geq BA_{\zdi_0, \zdi_1}(h)$, denote the \emph{optimal adversary}. 
Intuitively, $h^\star$ predicts $s$ for which the likelihood of $\vz$ under the corresponding distribution ($\zdi_0$ or $\zdi_1$) is larger. 
More formally, $h^\star(\vz) = \indicator{P(\vz | s=1) \geq P(\vz | s=0)}$ (see \eg \citet{FNF} for a proof).
As shown in \citet{feldman2015disparate, McNamara2017Provably, LAFTR},
\begin{equation} \label{eq:dp2ba}
    \Delta^{DP}_{\zdi_0, \zdi_1}(g) \leq 2 \cdot BA_{\zdi_0, \zdi_1}(h^\star) - 1
\end{equation} 
holds for any $g$, \ie $BA_{\zdi_0, \zdi_1}(h^\star)$ represents a fairness certificate.
Another similar upper bound noted by \citet{LAFTR,CondLFR,Shen2021FRL} uses the total variation distance $d_{TV}(\zdi_0, \zdi_1)$ \citep{LevinPeresWilmer2006} of distributions, as it closely relates to $h^\star$, namely $BA_{\zdi_0, \zdi_1}(h^\star) = 1/2 + 1/2 \cdot d_{TV}(\zdi_0, \zdi_1)$.
Utilizing these results to obtain certificates is not simple, as accurately approximating $h^\star$ or $d_{TV}(\zdi_0, \zdi_1)$ (or related distance measures from the family of IPM or $\phi$-divergences) is hard~\citep{tvd2,tvd3,tvd1}.

%% file: src/requirements.tex
\section{FRL with Practical Certificates} \label{sec:requirements}

We now list requirements that a practical certificate should satisfy, and reflect on prior attempts to achieve this goal.
 
\paragraph{Practical certificates}
We start by stating an informal definition of a \emph{practical DP distance certificate}: 
\begin{definition}[Informal]
    \label{def:practicalcert}
    Given small $\epsilon$, finite dataset $D=\{(\vx^{(j)}, s^{(j)})\}_{j=1}^n$ sampled from an unknown distribution $\mathcal{X}$, and an encoder $f$ mapping each $\vx^{(j)}$ into a corresponding representation $\vz^{(j)}$, a \emph{practical DP distance certificate} is a value $\cert \in \mathbb{R}$ such that $$\sup_{g \in \mathcal{G}} \Delta^{DP}_{\zdi_0, \zdi_1}(g) \leq \cert$$ holds with probability $1 - \epsilon$ (over the data sampling), where $\mathcal{G}$ is the set of all downstream classifiers. 
    In realistic use-cases, computing $\cert$ must lead to non-vacuous bounds, \ie $\cert < 1$.
\end{definition}

This definition naturally arises from the core principles of FRL, as discussed in \cref{sec:intro}, and is necessary to provide suitable assurances to involved parties.
We next formalize the requirements implicitly stated in~\cref{def:practicalcert} (denoted R1-R5), and discuss prior work, which has so far been unable to satisfy all requirements (see~\cref{app:survey} for a more detailed exposition). 
Further, we make the case that R1-R5 should be used as a first step (a \emph{sanity check}) in evaluating future attempts to obtain practical provable FRL.

We remark that prior work, while not solving the problem of practical certificates for FRL which we focus on, establishes useful theoretical results, proposes methods with good accuracy-fairness tradeoffs, or provides weaker assurances which may be sufficient for different use-cases than ours.

\paragraph{Nature of bounds}
A practical certificate should define a \emph{high-probability bound (R1)} that holds for a \emph{concrete sample of datapoints} $D$ collected by the data producer.
This is in contrast to expectation bounds (such as \eg~\citet{gitiaux}), which imply bounded unfairness of downstream classifiers only in expectation (with respect to data sampling), and do not provide \emph{any} assurance for a particular sample $D$ in the practical case we observe.

The certificate should also be a \emph{finite-sample bound (R2)}, as opposed to an asymptotic bound, which would hold only for $n \to \infty$, providing no strict assurance in a concrete case of fixed and limited $n$.
An example of this is a large set of works (\eg \citet{Shen2021FRL,FCRL,McNamara2017Provably,LAFTR}), where DP distance is soundly upper-bounded by a quantity not directly computable in practice (\eg $BA_{\zdi_0, \zdi_1}(h^\star)$), for which a more feasible approximation (\eg a 2-layer neural network modeling $h$) is then optimized in training (approximately, using SGD which generally has no guarantees)~\citep{McNamara2017Provably,LAFTR}.
As the quality of approximations is either not discussed or is of the asymptotic nature, these works can not offer practical certificates.

\paragraph{Limiting assumptions}
A practical certificate should be \emph{distribution-free (R3)}, not requiring any assumptions on the distribution $\mathcal{X}$, as otherwise the certificate may not hold for $\mathcal{X}$ which the data producer encounters in practice.
For instance, the certificates of \citet{fair-universal} are valid only if $\mathcal{X}$ is a mixture of Gaussians.
Similarly,~\citet{FNF} provide finite sample bounds that hold given high-confidence density estimation of $\mathcal{X}$, which is possible only under restrictive assumptions such as Lipschitz~\citep{kde1} or $\alpha$-Hölder continuity~\citep{kde2}. %

A practical certificate should also be \emph{model-free (R4)}, with no assumptions on the hypothesis class of the downstream classifiers $\mathcal{G}$. Having such assumptions overestimates fairness (see~\cref{sec:related}) and offers \emph{no protection} against data consumers using a model outside this family, an aspect which the data producer is by design unable to control.
For instance,~\citet{ObliviousKernels,KernelLFR} focus on providing certificates only for the family of Kernel models, which is an unrealistic restriction in our setting.

\paragraph{Empirical validation}
Finally, \cref{def:practicalcert} notes that a practical certificate should be \emph{empirically tested (R5)}, \ie it is important to demonstrate that for a realistic dataset and encoder, the certification procedure can return a real value $T^\star$, which is non-vacuous, \ie at most $1$, the maximum possible value of $\Delta^{DP}_{\zdi_0, \zdi_1}(g)$.
Unfortunately, most prior work (\citet{McNamara2017Provably,LAFTR}, etc.) violates R5, often producing certificates whose actual value can not be exactly computed in practice.
We strongly argue that a realistic empirical study is crucial to demonstrate the practical value of a proposed certification method.

The statistical procedure we will propose in~\cref{sec:proof} satisfies R1-R5: it produces $T^\star$, a high-probability finite-sample bound on $\Delta^{DP}_{\zdi_0, \zdi_1}(g)$, with no restrictive assumptions. We apply the proposed procedure in~\cref{sec:experiments} to several real datasets, obtaining decidedly non-vacuous values of $T^\star$, which are in some cases lower than the purely empirical unfairness values of some prior methods. 

%% file: src/proof.tex
\section{Deriving Practical Certificates for Restricted Encoders} \label{sec:proof}
   
We describe our key contribution, the derivation of practical fairness certificates for restricted encoders (formally defined shortly). In \cref{sec:tree} we instantiate this approach with a particular encoder based on fair decision trees.

Let $p_0$ and $p_1$ denote the PDFs of $\zdi_0$ and $\zdi_1$ respectively, \ie $p_0(\vz_i) = P(\vz_i | s=0)$ and $p_1(\vz_i) = P(\vz_i | s=1)$ and $p$ denote the PDF of the marginal distribution of $\vz$. Similarly, we use $q$ for the marginal distribution of $s$, and $q_i$ for the conditional distribution of $s$ for $\vz = z_i$, \ie $q_i(0) = P(s=0 | \vz = z_i)$ and $q_i(1) = P(s=1 | \vz = z_i)$.

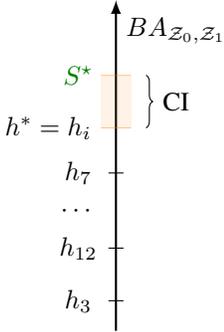
\begin{wrapfigure}{L}{0.17\textwidth}
    \vspace{1em}
    \centering 
    \input{figures/motivation}
    \caption{\label{fig:motivation} Restricted encoders enable practical certificates.}
  \end{wrapfigure}

\paragraph{Restricted encoders} 
As noted before, prior work is unable to directly utilize~\cref{eq:dp2ba} to obtain practical certificates, as for neural network encoders it is intractable to compute the densities $p_0(\vz)$ and $p_1(\vz)$ that define the optimal adversary $h^*$.
To remedy this, we propose using \emph{restricted encoders} \mbox{$f \colon \mathbb{R}^d \to \{\vz_1, \ldots, \vz_k\}$}, \ie encoders that map each $\vx$ to one of $k$ possible values (\emph{cells}) $\vz_i \in \mathbb{R}^{d'}$, \ie $\mathcal{Z}$ is a discrete distribution with a finite support.
We note that no prior work uses this concept---\citet{LFR} map data to \emph{prototypes}, but this mapping is probabilistic, thus incompatible with our definition.

As now there is a finite number of possible values for a representation, we can use a set $D$ of $n$ samples from $\zdi$ (obtained by applying $f$ to samples from $\xdi$) to analyze the optimal adversary $h^*$ on each possible $\vz$. Moreover, we can upper-bound its balanced accuracy (RHS of \cref{eq:dp2ba}) on the whole distribution $\zdi$ with some value $S^\star$ with high probability, using confidence intervals (CI, \cref{fig:motivation}). Finally, we can apply \cref{eq:dp2ba} to obtain the certificate \mbox{$\Delta^{DP}_{\zdi_0, \zdi_1}(g) \leq 2S^\star -1 = T^\star$}. 
As in \cref{def:practicalcert}, $T^\star$ has a dependency on $D$ and $n$, which we omit for brevity.
A detailed presentation of our certification procedure follows.

\paragraph{Upper-bounding the balanced accuracy} Starting from \cref{eq:bacc}, we reformulate the balanced accuracy of $h^\star$:
\begin{flalign}
    &{BA}_{\zdi_0, \zdi_1}(h^*) \nonumber \\
              &= \frac{1}{2} \left( \sum\limits_{i=1}^{k} p_0(\vz_i) \cdot [1-h^*(\vz_i)] + \sum\limits_{i=1}^{k} p_1(\vz_i) \cdot [h^*(\vz_i)] \right)   \nonumber \\ 
              &= \frac{1}{2} \left( \sum\limits_{i=1}^{k} \max\left( p_0(\vz_i), p_1(\vz_i) \right) \right) \nonumber \\ 
              &= \sum\limits_{i=1}^{k} p(\vz_i) \cdot \max \left(\underbrace{\frac{1}{2q(0)}}_{\alpha_0} \cdot q_i(0),\underbrace{\frac{1}{2q(1)}}_{\alpha_1} \cdot q_i(1) \right)\text{.} \nonumber
\end{flalign} 
The first line uses the definition of expectation of a discrete RV, the second the definition of $h^\star$, and the final line the two applications of Bayes' rule, namely $p_0(\vz_i) = {q_i(0) p(\vz_i)}/{q(0)}$ and $p_1(\vz_i) = {q_i(1) p(\vz_i)}/{q(1)}$.
To upper bound ${BA}_{\zdi_0, \zdi_1}(h^*)$ with high probability using given $n$ samples, we focus on the final expression above, the prior-weighted (\ie weighted by $p(\vz_i)$) per-cell balanced accuracy (\ie $\max(\alpha_0 q_i(0), \alpha_1 q_i(1))$ for each cell $i$), where we set $\alpha_0 = \frac{1}{2q(0)}$ and $\alpha_1 = \frac{1}{2q(1)}$.

Next, we introduce 3 lemmas, and later combine them to obtain the desired certificate. We use $B(p; v,w)$ to denote the $p$-th quantile of a beta distribution with parameters $v$ and $w$. Note that for \cref{lemma1} we do not use the values $\vz^{(j)}$ in the proof, but still introduce them for consistency.

\begin{lemma}[Bounding base rates] \label{lemma1}
    
Given $n$ independent samples $\{(\vz^{(1)}, s^{(1)}), \ldots, (\vz^{(n)}, s^{(n)})\} \sim \zdi$ and a parameter $\epsilon_{b}$, for $\alpha_0$ and $\alpha_1$ as defined above, it holds that
\begin{flalign}
    \alpha_0 &< \frac{1}{2B(\frac{\epsilon_b}{2}; m, n-m+1)}, \quad\text{and}\quad \nonumber \\ 
    \alpha_1 &< \frac{1}{2(1-B(1-\frac{\epsilon_b}{2}; m+1, n-m))} 
    \nonumber \text{,}
\end{flalign}
with confidence $1-\epsilon_b$, where $m=\sum_{j=1}^{n} \indicator{s^{(j)}=0}$.

\end{lemma}
\begin{proof}
    We define $n$ independent Bernoulli random variables  $X^{(j)} := \indicator{s^{(j)}=0}$ with same unknown success probability $q(0)$. Using the Clopper-Pearson binomial CI (\citet{CP}, see \cref{app:math}) to estimate the probability $q(0)$ we get 
    \mbox{$P(q(0) \leq B(\frac{\epsilon_b}{2}; m, n-m+1)) \leq \epsilon_b/2$}
    and
    \mbox{$P(q(0) \geq B(1-\frac{\epsilon_b}{2}; m+1, n-m)) \leq \epsilon_b/2$}. 
    Substituting $q(0)=1-q(1)$ in the latter, as well as the definitions of $\alpha_0$ and $\alpha_1$ in both inequalities recovers the inequalities from the lemma statement, which simultaneously hold with confidence $1-\epsilon_b$.
\end{proof}

\begin{lemma}[Bounding balanced accuracy for a cell] \label{lemma2}
    
Given $n$ independent samples $\{(\vz^{(1)}, s^{(1)}), \ldots, (\vz^{(n)}, s^{(n)})\} \sim \zdi$, constants $\bar{\alpha_0}$ and $\bar{\alpha_1}$ such that $\alpha_0 < \bar{\alpha_0}$ and $\alpha_1 < \bar{\alpha_1}$, and a parameter $\epsilon_c$, it holds for each cell $i \in \{1, \ldots, k\}$, with total confidence $1-\epsilon_c$, that
\begin{equation} \label{eq:cell}
    \max(\alpha_0 \cdot q_i(0), \alpha_1 \cdot q_i(1)) \leq t_i,
\end{equation} 
where $t_i$ is the maximum of ${\bar{\alpha_0}}{B(1-\frac{\epsilon_c}{2k}; m_i+1, n_i-m_i)}$ and ${\bar{\alpha_1}}(1-$ $B(\frac{\epsilon_c}{2k}; m_i, n_i-m_i+1))$. Here, $n_i = |Z_i|$, and $m_i = \sum_{j \in Z_i} \indicator{s^{(j)}=0}$, where $Z_i = \{j | \vz^{(j)} = \vz_i\}$.

\end{lemma}
\begin{proof}
  As in \cref{lemma1}, for each cell we use the Clopper-Pearson CI to estimate $q_i(0)$ with samples indexed by $Z_i$ and confidence $1 - \epsilon_c/k$. As before, we apply $q_i(0) = 1 - q_i(1)$ to arrive at $k$ inequalities of the form \cref{eq:cell}, which per union bound jointly hold with confidence $1-\epsilon_c$.
\end{proof}

\begin{lemma}[Bounding the sum] \label{lemma3}
Given $n$ independent samples $\{(\vz^{(1)}, s^{(1)}), \ldots, (\vz^{(n)}, s^{(n)})\} \sim \zdi$, where for each $j \in \{1, \ldots, n\}$ we define a function $idx(\vz^{(j)}) = i$ such that $\vz^{(j)} = \vz_i$ (cell index), parameter $\epsilon_s$, and a set of real-valued constants $\{t_1, \ldots, t_k\}$, it holds that 
\begin{flalign*}
    &P\left(\sum_{i=1}^k p(\vz_i) t_i \leq S^\star\right) \geq 1 - \epsilon_s, \quad\text{where~}\quad \\
    &S^\star = \frac{1}{n}\sum_{j=1}^{n} t_{idx(\vz^{(j)})} + (b-a) \sqrt{\frac{-\log \epsilon_s}{2n}},
\end{flalign*}
and we set $a = \min\{t_1, \ldots, t_k\}$ and $b = \max\{t_1, \ldots, t_k\}$.
\end{lemma}

\begin{proof} For each $j$ let $X^{(j)} := t_{idx(\vz^{(j)})}$ denote a random variable. As for all $j$, $X^{(j)} \in [a, b]$ with probability $1$ and $X^{(j)}$ are independent, we can apply Hoeffding's inequality (\citet{Hoeffding}, restated in \cref{app:math}) to upper-bound the difference between the population mean $\sum_{i=1}^k p(\vz_i) t_i = \E_{\vz \sim \mathcal{Z}} t_{idx(\vz)}$ and its empirical estimate $\frac{1}{n} \sum_{j=1}^{n} X^{(j)}$. Setting the upper bound such that the error is below $\epsilon_s$ directly recovers $S^\star$ and the lemma statement.
\end{proof}

\paragraph{Applying the lemmas} We now describe how we apply the lemmas in practice to
upper-bound $BA_{\zdi_0, \zdi_1}(h^\star)$, and in turn upper-bound $\Delta^{DP}_{\zdi_0, \zdi_1}(g)$ for any downstream classifier $g$.
We assume a standard setting, where a set $D$ of datapoints $\{(\vx^{(j)}, s^{(j)})\}$ from $\xdi$ is split into a training set $D_{train}$, used to train $f$, validation set $D_{val}$, held-out for the upper-bounding procedure (and not used in training of $f$ in any capacity), and a test set $D_{test}$, used to evaluate the empirical accuracy and fairness of downstream classifiers.

After training the encoder and applying it to produce representations  $(\vz^{(j)}, s^{(j)}) \sim \zdi$ for all three data subsets, we aim to derive an upper bound on $\Delta^{DP}_{\zdi_0, \zdi_1}(g)$ for any $g$, that holds with confidence at least $1-\epsilon$, where $\epsilon$ is the hyperparameter of the procedure (we use $\epsilon=0.05$).
To this end, we heuristically choose some decomposition $\epsilon = \epsilon_b + \epsilon_c + \epsilon_s$, and apply \cref{lemma1} on $D_{train}$ to obtain upper bounds $\alpha_0 < \bar{\alpha_0}$ and $\alpha_1 < \bar{\alpha_1}$ with error probability $\epsilon_b$. As mentioned above, using $D_{train}$ in this step is sound as estimated probabilities $q(0)$ and $q(1)$ are independent of the encoder $f$. Next, we use $\bar{\alpha_0}$, $\bar{\alpha_1}$ and $D_{val}$ in \cref{lemma2}, to obtain upper bounds $t_1, \ldots, t_k$ on per-cell accuracy that jointly hold with error probability $\epsilon_c$. Finally, we upper-bound the sum $\sum_{i=1}^k p(\vz_i) t_i \leq S^\star$ with error probability $\epsilon_s$ using \cref{lemma3} on $D_{test}$ with previously computed $t_1, \ldots, t_k$. Combining this with \cref{eq:dp2ba} finally gives the certificate
\begin{equation}
    \Delta^{DP}_{\zdi_0, \zdi_1}(g) \leq 2 \cdot BA_{\zdi_0, \zdi_1}(h^\star) - 1 \leq 2S^\star-1 = T^\star,
\end{equation}
which per union bound holds with desired error probability $\epsilon$, with respect to the  sampling process.

\paragraph{Practicality of the certificate} 
According to requirements in~\cref{sec:requirements}, our certificate is practical, matching~\cref{def:practicalcert}. 
We provided a high-probability (R1) finite-sample (R2) bound, with no restrictive assumptions on $\mathcal{X}$ (R3) or the family of downstream classifiers (R4). 
Our procedure can be applied to \emph{any} restricted $f$---as the certificate is derived for a fixed $f$, this is not an obstacle to certificate practicality, like R3 and R4.
We further hypothesize that for a suitable instantiation of restricted $f$ (\ie the one we introduce in \cref{sec:tree}, based on decision trees) the accuracy-fairness tradeoffs obtained are favorable, and that our procedure leads to non-vacuous certificates on real datasets.
We confirm this in our extensive experiments in~\cref{sec:experiments} (satisfying R5).

%% file: figures/motivation.tex
\begin{tikzpicture}

  \def\A{-1.8};
  \def\B{-1.1};
  \def\C{-0.1};
  \def\D{0.5};

  \def\P{0.7}; 
  \def\CI{0.7}; 

  \draw[-latex, thick] (0, -2.2) -- (0, 2.2);

  \node at (-0.5, \A) [] {$h_3$};
  \draw (-0.1, \A) -- (0.1, \A);

  \node at (-0.5, \B) [] {$h_{12}$};
  \draw (-0.1, \B) -- (0.1, \B);

  \node at (-0.5, 0.5*\B + 0.5*\C) [] {$\cdots$};
  
  \node at (-0.5, \C) [] {$h_7$};
  \draw (-0.1, \C) -- (0.1, \C);

  \node at (-0.9, \D) [] {$h^* = h_i$};

  \definecolor{threshgreen}{HTML}{008000}
  \node at (-0.5, \D + \CI) [] {$\textcolor{threshgreen}{S^\star}$};

  \node at (0.8, 1.8) [] {${BA}_{\zdi_0, \zdi_1}$};

  \path [fill=orange, opacity=0.1] (-0.2, \D) rectangle (0.2, \D + \CI);
  \draw[orange, opacity=0.4] (-0.2, \D) -- (0.2, \D);
  \draw[orange, opacity=0.4] (-0.2, \D + \CI) -- (0.2, \D + \CI);

  \draw [decorate, decoration = {brace,mirror}] (0.4, \D) --  (0.4, \D + \CI);
  \node at (0.8, \D + 0.5*\CI) [] {CI};

\end{tikzpicture}

%% file: src/tree.tex
\section{Fair Decision Trees as Restricted Encoders} \label{sec:tree}
In this section, we instantiate a restricted encoder which allows for favorable accuracy-fairness tradeoffs along with practical certificates.
We use decision trees, motivated by strong results of tree-based models on tabular data, their efficiency and interpretability \citep{tabular, treesnew}, real-life uses (\eg in finance~\citep{treez2,treez1}), and their feature splitting procedure, which allows control over the tightness of our certificate.

In~\cref{app:kmeans} we present an experiment with another class of restricted encoders based on k-means clustering.
This substantiates our claim that our procedure can directly produce practical certificates for \emph{any} restricted $f$, and illustrates that choosing a suitable restricted $f$ is hard, as fairness-unaware encoders (\eg k-means) likely lead to unfavorable results.
 
\paragraph{Classification trees}
Starting from the training set $D_{root}$ of examples $(\vx, y) \in \mathbb{R}^d \times \{0,1\}$, a binary \emph{classification tree} $f$ repeatedly \emph{splits} some leaf node $P$ with assigned $D_P$, \ie picks a split feature $j \in \{1, \ldots, d\}$
and a split threshold $v$, and adds two nodes $L$ and $R$ as children of $P$, such that \mbox{$D_L = \{(\vx, y) \in D_P ~|~ x_j \leq v \}$} and $D_R = D_P \setminus D_L$. $j$ and $v$ are picked to minimize a chosen criterion, weighted by $|D_L|$ and $|D_R|$, aiming to produce \mbox{$y$-homogeneous} leaves. We focus on Gini impurity, computed as $Gini_y(D) = 2p_y(1-p_y) \in [0, 0.5]$ where $p_y = \sum_{(\vx, y) \in D} \indicator{y = 1} / |D|$.
At test time, an example $\vx$ is propagated to a leaf $l$, predicting the majority class of $D_l$.

\input{figures/main3}
\paragraph{Trees as encoders}
Our key idea is to train a classification tree $f$ with $k$ leaves, encoding all examples in leaf $i$ to the same $\vz_i$. We construct $\vz_i$ based on all training examples in leaf $i$, taking the median for continuous, and the most common value for categorical features (so we have $d'=d$).

\paragraph{Fairness-aware criterion} Criteria such as $Gini_y(D)$ aim to maximize the accuracy by making the distribution of the label $y$ in each leaf highly unbalanced. Using such a tree as an encoder leads to high unfairness, making it necessary to introduce a way to prioritize more fair tree structures.
 
We use, similar to \citet{Kamiran} and others (see \cref{sec:related}), the criterion $FairGini(D) = (1-\gamma) Gini_y(D) + \gamma (0.5 - Gini_s(D)) \in [0, 0.5]$. The second term aims to \emph{maximize} $Gini_s(D)$, \ie make the distribution of $s$ in each leaf $i$ as close to uniform (making it hard to infer $s$ from $\vz_i$), while $\gamma$ controls the accuracy-fairness tradeoff.

\paragraph{Fairness-aware categorical splits} While usual splits of the form $x_j \leq v$ are suitable for continuous, they are inefficient for categorical (usually one-hot) variables, as only $1$ category can be isolated. This increases the number of cells, making our certificates loose. Instead, we represent $n_j$ categories for feature $j$ as integers $ c\in \{1, 2, ..., n_j\}$. To avoid evaluating all \mbox{$2^{n_j}-1$} possible partitions, we sort the values by $p_y(c) = \sum_{(\vx, y) \in D_c} \indicator{y = 1} / |D_c|$ where we set \mbox{$D_c = \{\vx \in D~|~ \vx_j=c\}$}, and consider all prefix-suffix partitions (\emph{Breiman shortcut}).

This ordering focuses on accuracy and is provably optimal for $FairGini(D)$ with $\gamma=0$ \citep{Breiman}. However, as it ignores fairness, it is inefficient for $\gamma > 0$. To alleviate this, we generalize the Breiman shortcut, and explore all prefix-suffix partitions under several orderings. Namely, for several values of the parameter $q$, we split the set of categories $\{1, 2, \ldots, n_j\}$ in \mbox{$q$-quantiles} with respect to $p_s(c)$ (defined analogous to $p_y(c)$), and sort each quantile by $p_y(c)$ as before, interspersing $q$ obtained arrays to obtain the final ordering. While this offers no optimality guarantees, it is an efficient way to consider both objectives.

%% file: figures/main3.tex
\begin{figure*}[t]
    \newcommand{\relSubfigWidth}{0.42\textwidth}
    \newcommand{\innerWidth}{\textwidth}
    \centering
    \hfill
    \begin{subfigure}{\relSubfigWidth}
        \centering 
        \includegraphics[width=\innerWidth]{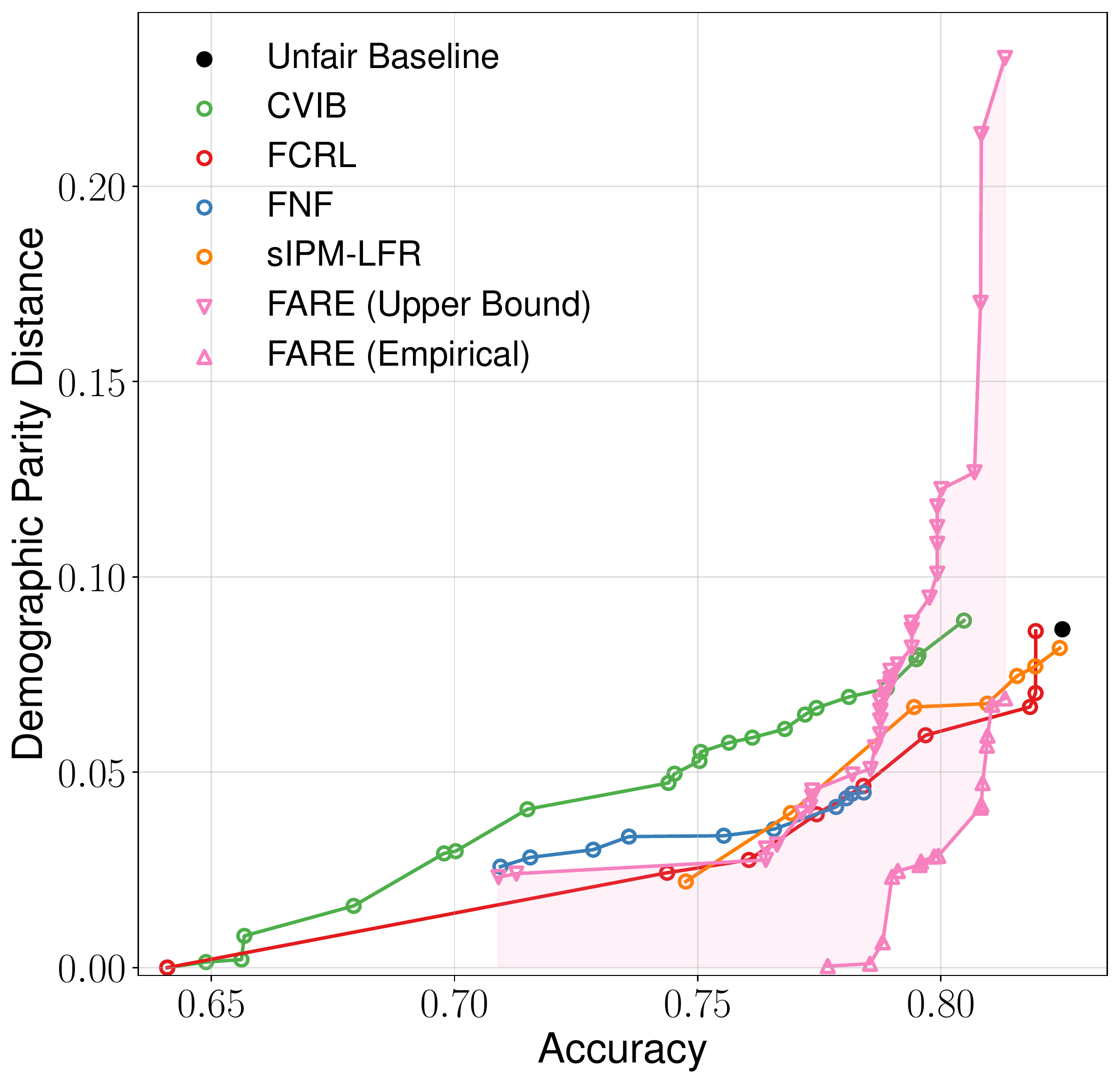}
    \end{subfigure}
    \hfill
    \hspace{1.8em}
    \begin{subfigure}{\relSubfigWidth}
      \centering
      \includegraphics[width=\innerWidth]{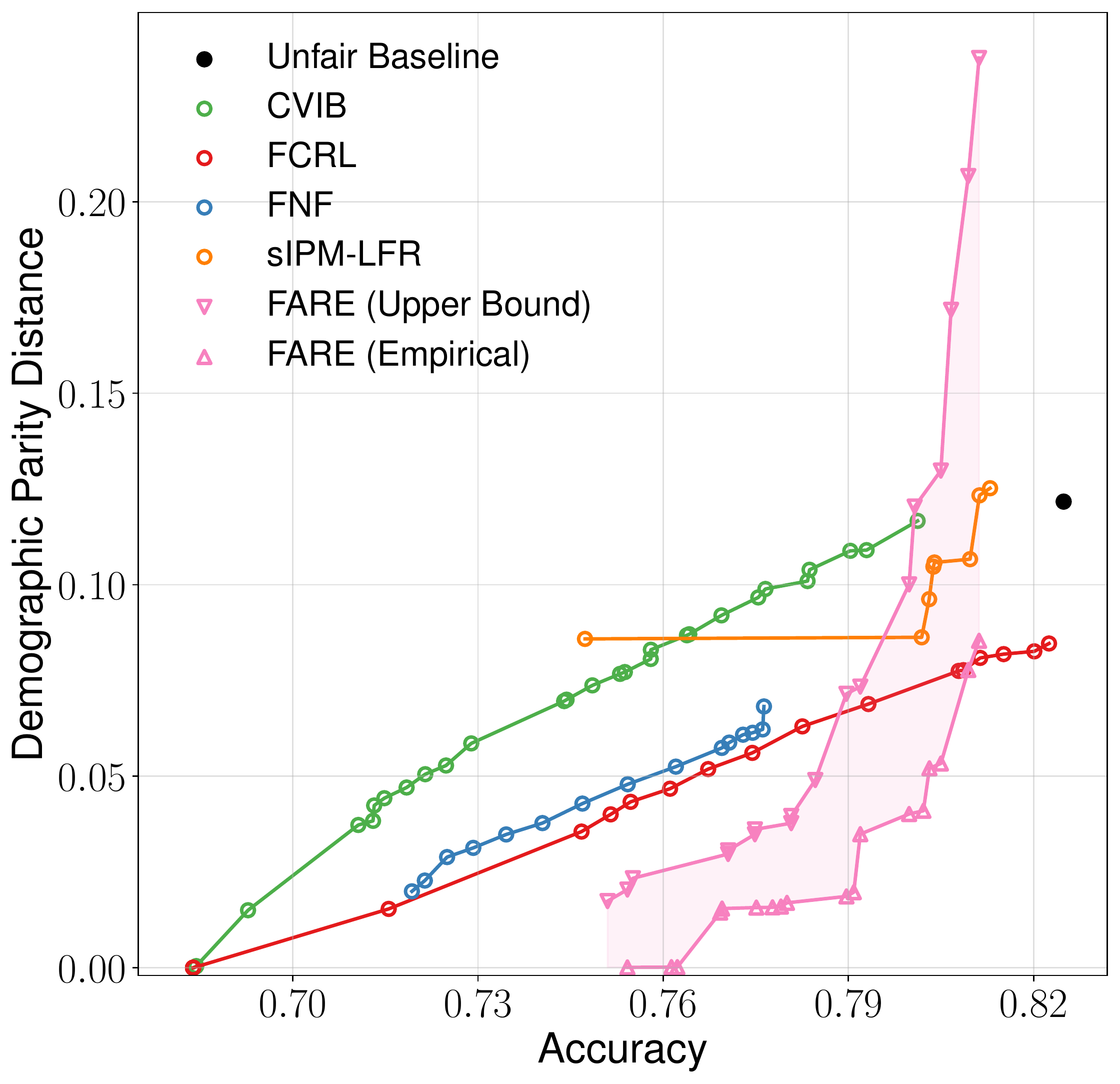}
    \end{subfigure}
    \hspace{1.5em}
    \hfill
    \caption{Evaluation of fair representation learning methods on ACSIncome-CA (left) and ACSIncome-US (right) datasets.}
    \label{fig:MainACS}
\end{figure*} 

\begin{figure}[t]
  \includegraphics[width=0.42\textwidth]{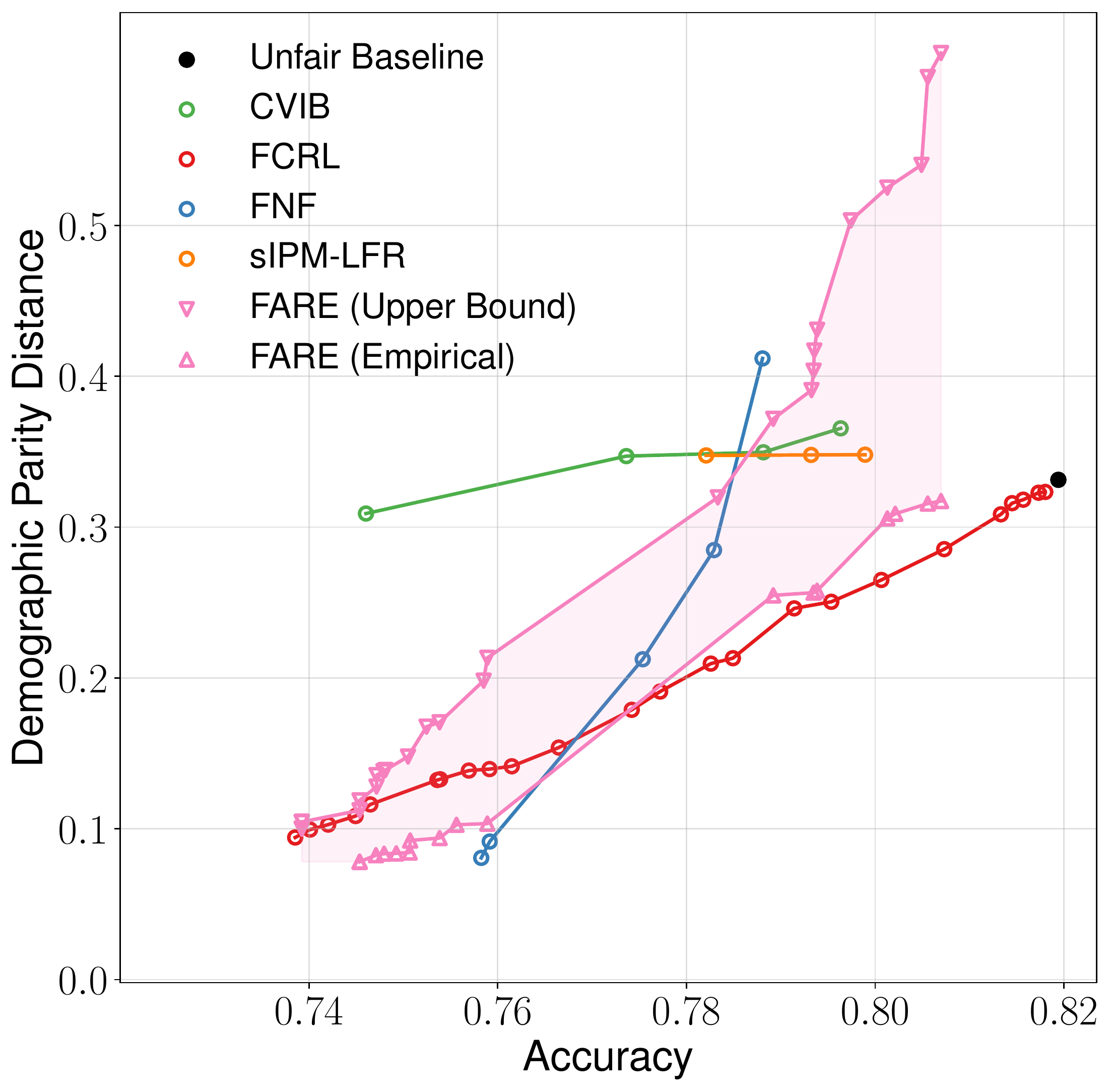}
  \centering
  \caption{Evaluation of FRL methods on the Health dataset.}
  \label{fig:MainHealth}
\end{figure} 

%% file: src/experiments.tex
\section{Experimental Evaluation} \label{sec:experiments}
We evaluate FARE on several datasets, showing that it produces representations with fairness-accuracy tradeoffs comparable to prior work, while for the first time offering practical certificates.
We then present several additional studies.

\paragraph{Experimental setup}

We consider common fairness datasets: Health~\citep{health}, ACSIncome-CA (only California), and ACSIncome-US (US-wide)~\citep{ding2021retiring}. The sensitive attributes are age and sex, respectively.
We include the following FRL baselines: LAFTR~\citep{LAFTR}, CVIB~\citep{CVIB}, FCRL~\citep{FCRL}, FNF~\citep{FNF}, sIPM~\citep{sIPM}, and FairPath~\citep{fairpath}. 
All omitted details regarding our experiments are given in~\cref{app:expdetails}.

\paragraph{Main experiments} We explore the fairness-accuracy tradeoff of each method by running it with various hyperparameters to obtain different representations, further used to train a 1-hidden-layer neural network (1-NN, other classifiers explored in \cref{app:moreexp:downstream}) for a certain prediction task, whose DP distance and prediction accuracy are then plotted.
All hyperparameters of FARE are listed in \cref{app:expdetails} and we present an ablation study of the key parameter $\bar{k}$ in~\cref{app:moreexp:ablation}, which can be used alongside our main results to guide parameter choices in practical applications of FARE.
 
We show a test set Pareto front for each method. 
For FARE, we also show a Pareto front of the certificate, \ie a 95\% confidence provable upper bound on DP distance (results with other metrics are deferred to~\cref{app:ssec:metrics} and lead to similar conclusions).
As noted before, no other method provides practical certificates that could be compared to FARE's.  
Finally, we include the Unfair Baseline, \ie an identity encoder.
The results are shown in~\cref{fig:MainACS} and~\cref{fig:MainHealth}. We omit FairPath and LAFTR from the main plots (see extended results in \cref{app:extendedresults}), as LAFTR has stability issues~\citep{FCRL, sIPM}, and FairPath uses a different metric to us~\citep{fairpath}.

FARE achieves a better or comparable accuracy-fairness tradeoff compared to baselines.
Crucially, the baselines cannot guarantee that there is no classifier with a worse DP distance when trained on their representations.
This cannot happen for FARE---we compute a \emph{provable} upper bound on DP distance of \emph{any} such classifier.
Our certificate is often comparable to \emph{empirical} values of baselines.
We note a small gap ($\leq 1.5\%$) between the best accuracy of FARE and the unfair baseline, indicating a tradeoff of restricted encoders---FARE computes a practical certificate, but loses some information, limiting the predictive power of classifiers.
This is rarely a practical issue, as achieving meaningful fairness generally requires a non-trivial accuracy loss, especially when $s$ and $y$ are correlated.
In \cref{app:moreexp:larger} we show that this gap is unaffected by dataset size.
Finally, another important advantage of FARE is efficiency, with runtime of only several seconds, as opposed to minutes or hours for all other baselines (we measure this in~\cref{app:efficiency}).

\paragraph{Data improves bounds} 
In \cref{fig:MainACS} we see that FARE certificates are tighter for ACSIncome-US, suggesting that using more samples improves the bounds.
To investigate this further in a controlled manner, we choose a representative set of FARE points from \cref{fig:MainACS} (left), and repeat the certification procedure with the dataset repeated $M$ times, showing the resulting upper bounds for $M \in \{2,4,8,16,32\}$ in \cref{fig:bigdata}.
We observe a significant improvement in the certificate for larger $M$, further reinforcing our intuition that FARE is well-suited for large datasets and will benefit from ever-increasing amounts of data used in ML~\citep{IncreasingData}. 
Note that the bounds obtained for some $M>1$ are only valid for the corresponding duplicated dataset (where we assume it was directly sampled), and do not hold on the original one, \ie we cannot simply boost bound tightness on the original problem via data duplication. In \cref{app:moreexp:shiftimpute} we study the robustness of FARE to distribution shifts and missing data, and in~\cref{app:moreexp:sens} show that it is robust to sensitive attribute imbalance.

\begin{figure}[t]
    \centering
    \includegraphics[width=0.42\textwidth]{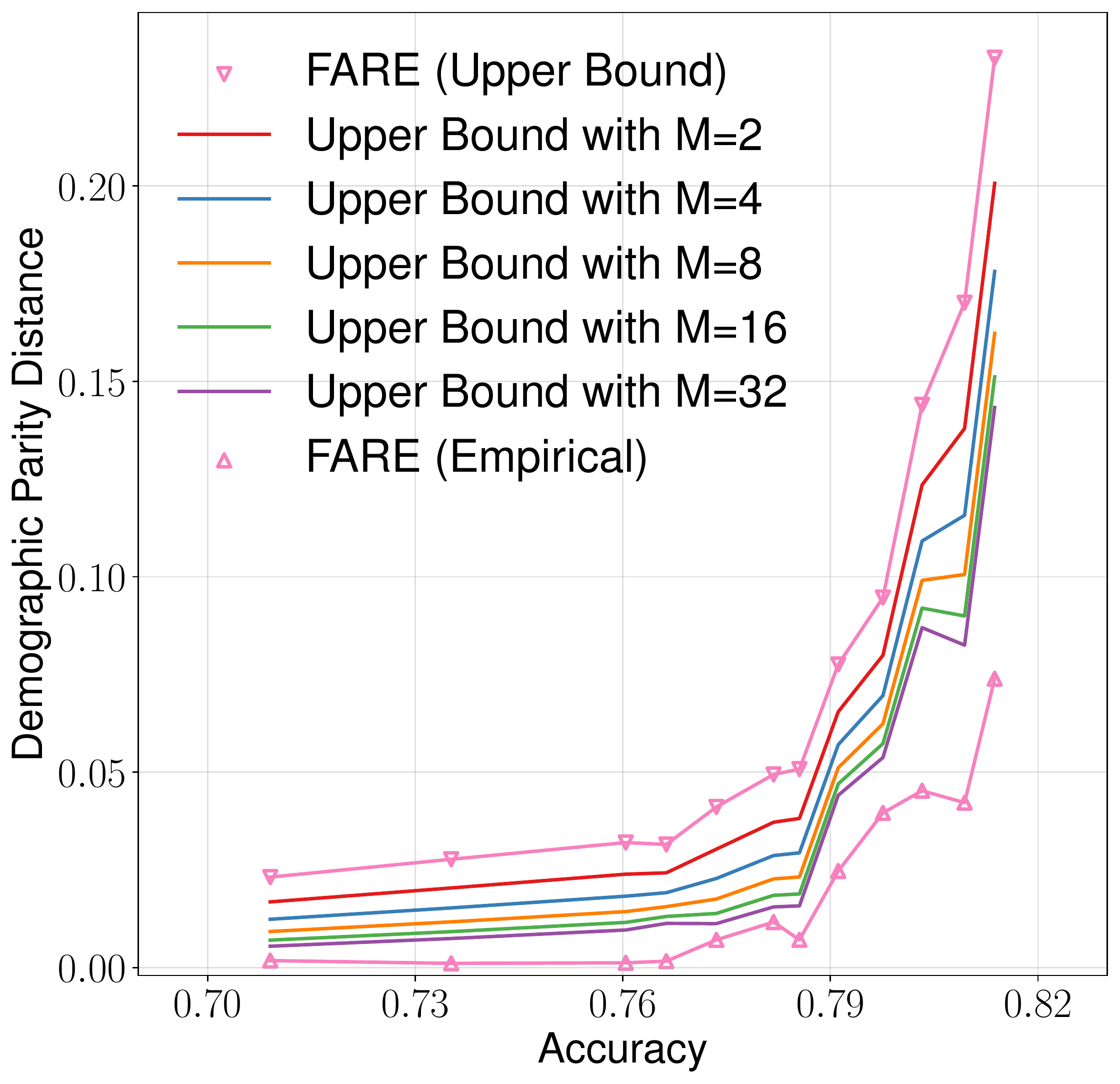}
    \caption{The impact of increasing the dataset size $M$ times on the tightness of the FARE fairness certificate.}
    \label{fig:bigdata}
\end{figure}
\begin{wrapfigure}[17]{L}{0.27\textwidth}
    \centering
    \includegraphics[width=0.27\textwidth]{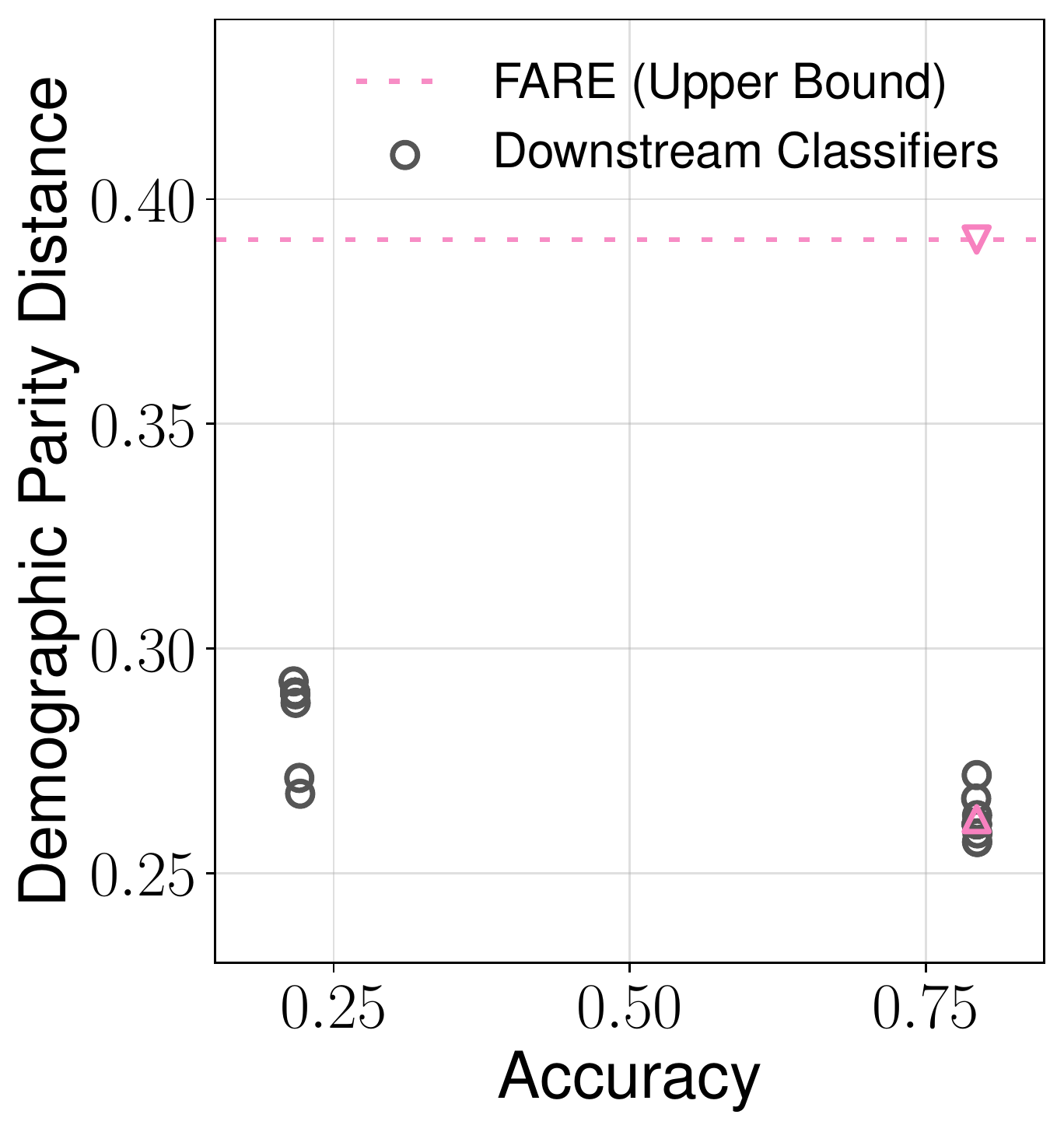}
    \caption{Downstream classifiers are below the FARE certificate.}
   \label{fig:downstream_main}
\end{wrapfigure} 

\paragraph{Certificate validity}
In the next experiment we investigate a representative point obtained by FARE from \cref{fig:MainHealth} with accuracy 79.3\%, empirical DP of 26.2\% and DP upper bound of 39.1\%.
In \cref{fig:downstream_main}, we show this point together with 24 other downstream classifiers (see \cref{app:expdetails}) obtained by training diverse model classes on the same representations (simulating a real use case where each data consumer might prefer to use a different model).
Half of the models are trained to maximize unfairness, and half to maximize accuracy.
We make two observations.
The left cluster shows that the data consumer can train models with high unfairness, especially in the worst case when they intentionally discriminate.
The right cluster reaffirms a known limitation of prior work \citep{xu2020theory,FCRL}: evaluating representations with some model class (triangle shows the 1-NN model from \cref{fig:MainHealth}) can underestimate unfairness, as data consumers using a different model class (\eg SVM) can be more unfair.
Both of these highlight the value of FARE which provides a practical certificate---all unfairness values still remain below a known upper bound.

\paragraph{Interpretability}
Another advantage of FARE is that its tree-based encoder enables direct interpretation of representations.
To illustrate this, we explore a point from \cref{fig:MainACS} (right), with accuracy $75.1\%$ and DP distance of $0.005$, where we have $k=6$.
We can easily find that, for example, the representation $\vz_6$ is assigned to each person older that $24$, with at least a Bachelor's degree, and an occupation in management, business or science.
We note that this illustrative point has a particularly low value of $k$, greatly sacrificing accuracy for fairness---other Pareto front points often use larger values of $k$ (see \cref{app:expdetails}).

\paragraph{Transfer learning}
Finally, we analyze the transferability of learned representations across tasks.
We produce a diverse set of representations on the Health dataset with each method, and following the procedure from prior work~\citep{LAFTR, FNF, sIPM} evaluate them on five unseen tasks $y$ (see \cref{app:expdetails} for details), where for each the goal is to predict a certain primary condition group.
For each task and each method, we identify the highest accuracy obtained while keeping $\Delta^{DP}_{\zdi_0, \zdi_1}$ not above 0.20 (or 0.05).
Moreover, we show $T^\star$, the provable DP distance upper bound of FARE.

The results are shown in~\cref{tab:transfer}.
We observe that some methods are unable to reduce $\Delta^{DP}_{\zdi_0, \zdi_1}$ below the given threshold.
FARE can always reduce $\Delta^{DP}_{\zdi_0, \zdi_1}$ sufficiently, but due to our restriction which enables the practical certificate, we often lose more accuracy for high $\Delta^{DP}_{\zdi_0, \zdi_1}$ thresholds.

\input{figures/transfer}

%% file: figures/transfer.tex
\begin{table}
    \centering
    \resizebox{0.41\textwidth}{!}{
    \begin{tabular}{c|c|ccccc}
    \toprule
    $y$ & $\Delta^{DP}_{\zdi_0, \zdi_1}$ & $T^\star$ & FARE & FCRL & FNF & sIPM \\
      \midrule
      \multirow{2}{*}{MIS} & $\leq$ 0.20 & 0.73 & 79.3 & 78.6 & 78.9 & 79.8 \\
       & $\leq$ 0.05 & 0.57 & 78.7 & 78.6 & 78.7 & 78.6 \\
      \midrule
      \multirow{2}{*}{NEU} & $\leq$ 0.20 & 0.72 & 73.2 & 72.4 & 71.9 & 76.6 \\
       & $\leq$ 0.05 & 0.43 & 72.1 & 71.4 & 71.7 & /\\
      \midrule
      \multirow{2}{*}{ART} & $\leq$ 0.20 & 0.55 & 74.4 & 70.7 & 68.9 & 78.3 \\
       & $\leq$ 0.05 & 0.12 & 69.5 & 69.5 & 68.5 & /\\
      \midrule
      \multirow{2}{*}{MET} & $\leq$ 0.20 & 0.48 & 69.8 & 69.2 & 75.0 & /\\
       & $\leq$ 0.05 & 0.12 & 66.1 & 65.3 & /& /\\
      \midrule
      \multirow{2}{*}{MSC} & $\leq$ 0.20 & 0.48 & 67.4 & 70.5 & 73.0 & /\\
       & $\leq$ 0.05 & 0.12 & 63.0 & /& /& /\\
    \bottomrule
    \end{tabular}}
    \caption{The results of transfer learning on Health.}
    \label{tab:transfer}
  \end{table} 

%% file: src/limitations.tex
\section{Limitations and Future Work}
Here we reflect on the limitations of FARE and highlight interesting avenues for future work.
While FARE is the first FRL method to provide practical certificates, and its empirical tradeoffs are generally favorable, our results in~\cref{sec:experiments} illustrate two main directions in which its results could be improved: (i) investigating and reducing the gap between the best achievable accuracy and the unfair baseline, (ii) tightening the certificate in high-accuracy regions. 
To this end, future work may attempt to extend the tree-based instantiation, or look for other more suitable instantiations of restricted encoders, analyzing their fundamental tradeoffs.
Further, an important step for future work is the improvement of transfer learning results which are especially relevant in the FRL setting.
Finally, important settings that we briefly study such as multivalued $y$ and $s$ (\cref{app:ssec:multi}) or robustness to distribution shifts (\cref{app:moreexp:shiftimpute}) could be studied in more detail in the context of restricted encoders. 

%% file: src/conclusion.tex
\section{Conclusion} \label{sec:conclusion}

We introduced FARE, a method for provably fair representation learning with practical certificates. The key idea was that using restricted encoders enables a practical statistical procedure for computing a high probability finite-sample upper bound on the unfairness of any downstream classifier. We instantiated this idea with a tree-based encoder, and experimentally demonstrated that FARE can, for the first time, obtain tight fairness bounds on several datasets, while simultaneously producing empirical fairness-accuracy tradeoffs similar to prior work which offers no practical guarantees. 

\section*{Acknowledgements}
We thank Angéline Pouget and Nikola Konstantinov for helpful feedback on previous versions of this paper.
This work has received funding from the Swiss State Secretariat for Education, Research and Innovation (SERI) (SERI-funded ERC Consolidator Grant).

\vfill 
\clearpage

%% file: src/appendix.tex
\vfill
\pagebreak

\input{src/app_survey}
\input{src/app_math}
\input{src/app_kmeans}
\input{src/app_generalizations}
\input{src/app_expdetails}
\input{src/app_extendedresults}

\input{src/app_efficiency}
\input{src/app_moreexperiments}

%% file: src/app_survey.tex
\section{Prior Work on Provably Fair Representation Learning} \label{app:survey}

Here we describe prior work on FRL aiming to produce provable guarantees.
Works that solely propose FRL methods or consider other theoretical aspects were covered in~\cref{sec:related}.
We remark that most works, in addition to the requirement from~\cref{sec:requirements} we will note is violated, usually also violate R5, providing no empirical evaluation of their certificate, often due to it not being exactly computable in practice. 

The work of~\citet{feldman2015disparate} was the first to establish the link between the balanced accuracy of the optimal adversary and fairness (which we have discussed in \cref{sec:background}) but for the case of disparate impact (the $80\%$ rule). They use this to motivate their method, which is in the category of in-processing methods (SVM training) and thus not FRL. \citet{McNamara2017Provably} further establish the same link for demographic parity, as well as a similar formulation based on entropy. Their adversarial training method is inspired by this but states no formal certificate. \citet{LAFTR} provides proof of similar results for more metrics, and discusses the relationship to total variation distance. As other works, they use this to motivate their adversarial training approach, where, as discussed in~\cref{sec:requirements} they approximately
(with stochastic gradient descent) optimize an approximation of an uncomputable certificate. 
\citet{CondLFR} also show that TV distance can bound fairness metrics but focus on simultaneously achieving equalized odds and accuracy parity, showing that perfect equalized odds imply bounds on demographic parity; their method is, similar to previous works, a min-max approximate optimization with no strict certificate.
None of these works thus provide a practical certificate, at best violating requirements R1, R2 and R5.

The work of \citet{Shen2021FRL} analyzes transferability between 7 group fairness notions under perfect or approximate fairness and discriminativeness. However, their method aims to minimize a finite sample estimate of MMD, which is introduced as an asymptotic approximation to TV distance (violating R2). As noted in~\cref{sec:requirements}, while reporting the obtained certificate (which the authors do) may be useful in some cases, the certificate is not valid in a strict sense, as the approximations are not accounted for.

\citet{fair-universal} state that perfect total variation distance implies perfect demographic parity. Further, they formulate the optimal adversary for each type of the reconstruction loss, but then note that it is not possible to compute those in practice (except for the restricted case when input distribution is \eg a mixture of Gaussians, violating R3), and they simply use it as a motivation for min-max optimization of an adversary in training.

\citet{sIPM} consider the family of integral probability metrics (IPM), and restrict the witness function class (compared to TV distance) to obtain a SigmoidIPM metric. As noted above, this effectively restricts the domain of classifiers (and is thus not model-free, violating R4) that the certificate applies to. The authors note this, and provide some classes for which the certificate would hold, however under the assumption of SGD optimality.

The works of~\citet{CVIB,MIFR,FCRL} all focus on the unfairness bounds via mutual information, stating bounds on this uncomputable quantity and approximately optimizing Monte Carlo estimates, thus as noted above, producing no strict certificate.
The missing further analysis of the approximations would likely lead to violations of R1 or R2.

Several works study provable FRL restricted to kernel models, thus not model-free (violating R4). \citet{donini2018erm} study in-processing with kernel-based methods, and their results apply to FRL (as noted) only in the special case of linear models. \citet{KernelLFR} provide model-aware bounds for kernel models under mild further assumptions, similar to~\citet{ObliviousKernels} who optimize a relaxation of the MMD metric. 
All these methods are restricted by their dependency on the model class, as kernel models, despite having specific use-cases, are generally unable to obtain state-of-the-art results for common problems of interest, and are often not scalable to complex data.

\citet{feng2019adv} bound the accuracy of the optimal Lipschitz-continuous adversary (though the general optimal adversary does not have to be Lipschitz continuous, thus violating R4) using Wasserstein distance (an instance of the IPM) and the Lipschitz constant. They perform adversarial training with K-Lipschitz adversaries, however it is not possible to exactly compute the Wasserstein distance in practice, and produce the certificate.

\citet{gitiaux} propose to add Gaussian noise to representations to obtain a finite sample two-sided expectation bound (violating R1) on how well DP distance can be approximated. The bound depends on the $l_\infty$ norm of the encoder (obtained by upper bounding and approximating the $\xi^2$ mutual information). As computing this norm is NP-hard~\cite{triki}, this is also not computable, making it practically impossible to satisfy R5.

Finally,~\citet{FNF} propose an FRL method where the encoder is a normalizing flow, which allows computation of probability densities in the latent space from densities in the input space, and bounding of the total variation distance. 
Under the assumption that the input density is known, this provides a practical certificate.
However, as noted in~\cref{sec:requirements}, for this certificate to be valid it is necessary to estimate densities with high confidence, which is generally only feasible under restrictive assumptions on the distribution, which in turn violates R3.

\paragraph{Other theoretical contributions}
Finally, it is worth mentioning that a line of theoretical work studies other aspects of FRL besides designing fairness certificates, such as tradeoffs between different notions, fairness-utility tradeoffs, and impossibility results~\citep{kleinberg,mcnamara, TheoryLechner21, TheoryZhao19}.

%% file: src/app_math.tex
\section{Mathematical Details} \label{app:math}

We first derive~\cref{eq:dp2ba}. More details can be found in \citet{LAFTR}, and here we provide an overview:

\begin{align*}
    \Delta^{DP}_{\mathcal{Z}_0, \mathcal{Z}_1}(g) &= \left|\E_{z \sim \mathcal{Z}_0}[g(z)] - \E_{z \sim \mathcal{Z}_1}[g(z)]\right| \\
&= \left|\E_{z \sim \mathcal{Z}_0}[-g(z)] + \E_{z \sim \mathcal{Z}_1}[g(z)]\right| \\
&= \left|\E_{z \sim \mathcal{Z}_0}[1-g(z)] + \E_{z \sim \mathcal{Z}_1}[g(z)] - 1 \right| \\
&= \left|2{BA}_{\mathcal{Z}_0, \mathcal{Z}_1}(g) - 1 \right|
\end{align*}

From this, we can argue that we can drop the absolute value and bound the balanced accuracy of $g$ with the balanced accuracy of $h^*$, finally arriving at~\cref{eq:dp2ba}.

Then, we formally state the Hoeffding's inequality and the Clopper-Pearson binomial confidence intervals, used in our upper-bounding procedure in \cref{sec:proof}.

\emph{Hoeffding's inequality \citep{Hoeffding}}: Let $X^{(1)}, \ldots, X^{(n)}$ be independent random variables such that \mbox{$P(X^{(j)} \in [a^{(j)}, b^{(j)}]) = 1$}. Let $\hat{\mu}=\frac{X^{(1)} + \ldots X^{(n)}}{n}$ and $\mu = \E[\hat{\mu}]$. It holds that: 
\begin{equation} \label{eq:hoeffding}
    P(\mu - \hat{\mu} \geq t) \leq \exp\left( \frac{-2n^2t^2}{\sum_{i=1}^n (b^{(i)} - a^{(i)})^2} \right). \nonumber 
\end{equation} 

\emph{Clopper-Pearson binomial proportion confidence intervals \citep{CP}}: Assume a binomial distribution with an unknown success probability $\theta$. Given $m$ successes out of $n$ experiments, it holds that:
\begin{equation} \label{eq:cp}
  B(\frac{\alpha}{2}; m, n-m+1) < \theta < B(1 - \frac{\alpha}{2}; m+1, n-m) 
\end{equation}
with confidence at least $1-\alpha$ over the sampling process, where $B(p; v,w)$ denotes the $p$-th quantile of a beta distribution with parameters $v$ and $w$. The Clopper-Pearson confidence interval has symmetric coverage probability, \ie each side of \cref{eq:cp} holds with confidence $1-\alpha/2$.

%% file: src/app_kmeans.tex
\section{K-means Restricted Encoders} \label{app:kmeans}

To substantiate our key claim from~\cref{sec:proof} that our statistical procedure can be applied to any restricted encoder, we consider encoders based on \emph{k-means clustering}---\ie for a given $k$, we cluster the input representations, and map all examples from the same cluster to the respective cluster center. 
This fits our definition of restricted encoders given in~\cref{sec:proof}, as the distribution $\mathcal{Z}$ has finite support.

\paragraph{Results}
For $k \in [2, 500]$, in \cref{fig:kmeans} we report the DP distance (empirical and upper bound, \ie the certificate) of the k-means restricted encoder on the ACSIncome-CA dataset, alongside the results of FARE from \cref{fig:MainACS} (left).
While we were able to directly apply our statistical procedure to obtain upper bounds on unfairness, as k-means encoders are fairness-unaware, both empirical results and upper bounds are unfavorable and greatly outperformed by FARE. 
This illustrates that finding suitable classes of restricted encoders is not simple, and highlights the problem of finding other well-performing encoder classes.

\begin{figure}[t]
    \includegraphics[width=0.38\textwidth]{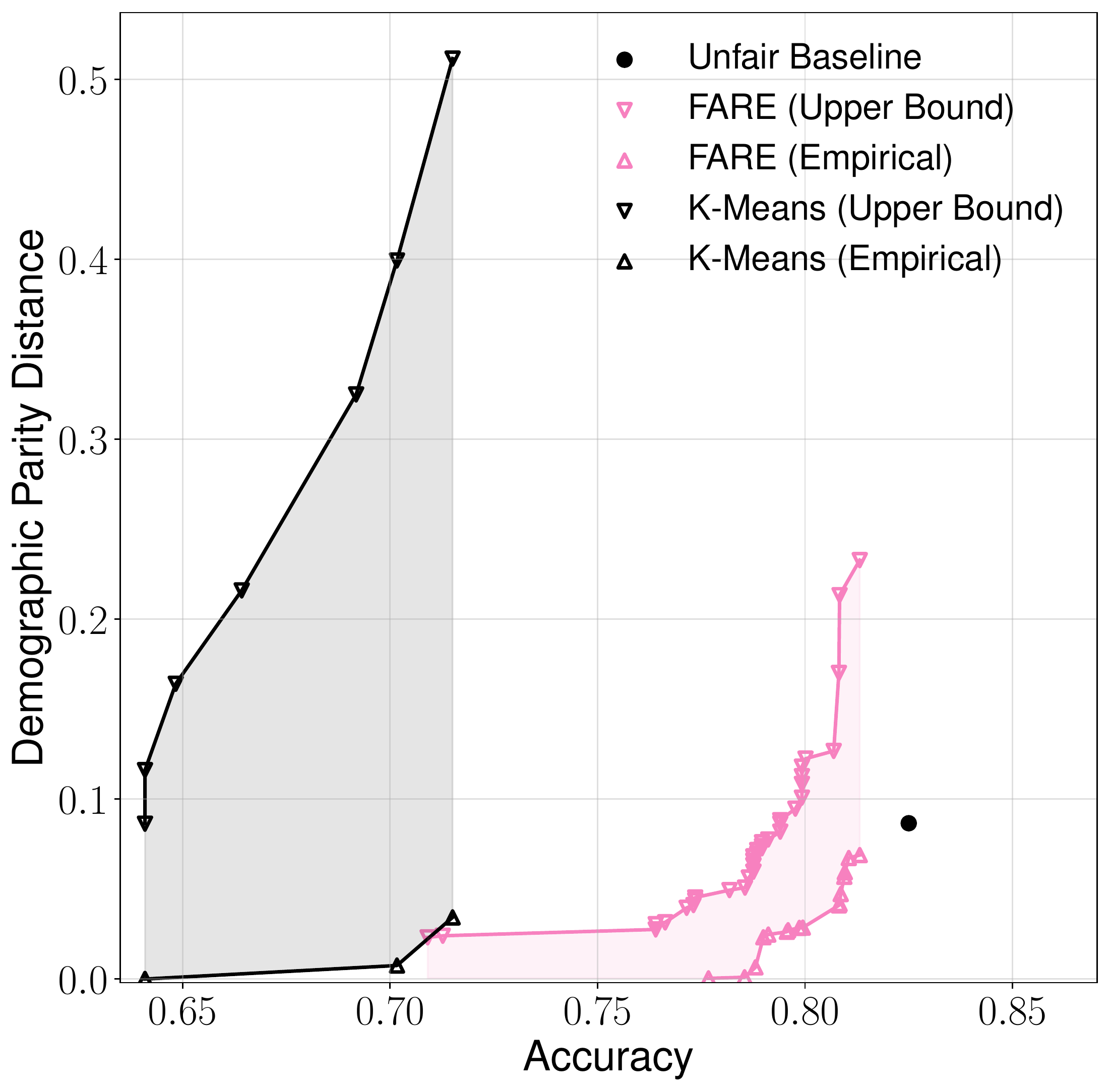}
    \centering
    \caption{Comparing the k-means restricted encoders and FARE on ACSIncome-CA.}
    \label{fig:kmeans}
  \end{figure} 

%% file: src/app_generalizations.tex
\section{Generalizations of FARE} \label{app:generalizations}

In this section we discuss the generalizations of FARE beyond the settings considered in the main paper. 
In~\cref{app:ssec:multi} we focus on the relaxation of the requirements of binary classification and binary sensitive attribute, and in~\cref{app:ssec:metrics} on additional unfairness metrics.

\subsection{Beyond binary sensitive attributes and labels} \label{app:ssec:multi}

We now describe how to extend FARE to the case of multivalued $s$ and $y$, demonstrating that our method is not fundamentally limited to the binary setting.
We use a common way to generalize DP distance, by considering the \emph{maximal} absolute difference in prediction rates of \emph{any} class $y$, w.r.t. \emph{any} two sensitive groups $i$ and $j$.

First, notice that our current certificate holds for any cardinality of $y$. 
Namely, we bound the unfairness of any binary classifier, and it can be easily seen that the most unfair classifier will always be binary (\ie will only predict 2 distinct classes, as this always maximizes the difference in prediction rates).
Next, for non-binary $s$, we can simply invoke the same procedure for each pair of sensitive groups $(i, j)$ by providing the procedure with only samples from those groups. 
Reducing $\epsilon$ in each of these invocations $\binom{|s|}{2}$ times leads to the same $1-\epsilon$ confidence for our certificate as before. 
Regarding our training procedure, while in the main paper we state the binary formulation of Gini impurity, the original definition directly supports multiple values of $y$/$s$, thus FARE training is already applicable to the general case.

\paragraph{Results}
To demonstrate that FARE can be applied to this setting on a real example, we provide preliminary results on a modified version of the ACSIncome-CA dataset, which represents 4-class classification (income classes thresholded at $[20k, 50k, 100k]$ dollars) and uses 3 sensitive groups (instead of sex we use a coarsening of race). 
We run FARE with above changes on this dataset and compute the certificate; we naively use the same hyperparameters as for ACSIncome-CA.
The results are shown in~\cref{fig:multi}.

\begin{figure}[t]
    \centering
    \includegraphics[width=0.38\textwidth]{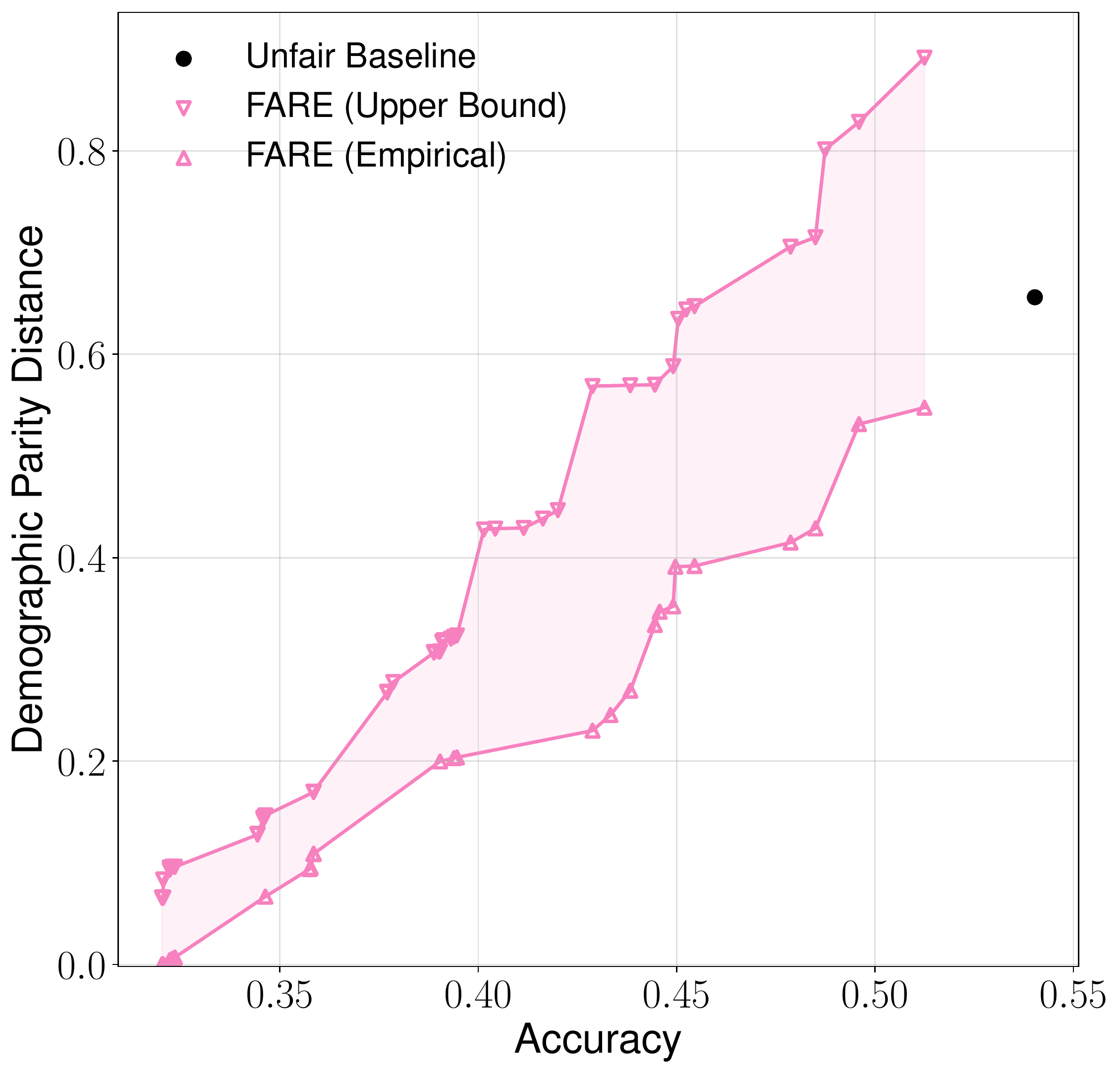}
    \caption{FARE can be generalized to multivalued $y/s$.}
    \label{fig:multi}
\end{figure}

\begin{figure*}[t]
    \newcommand{\relSubfigWidth}{0.42\textwidth}
    \newcommand{\innerWidth}{\textwidth}
    \centering
    \hfill
    \begin{subfigure}{\relSubfigWidth}
        \centering 
        \includegraphics[width=\innerWidth]{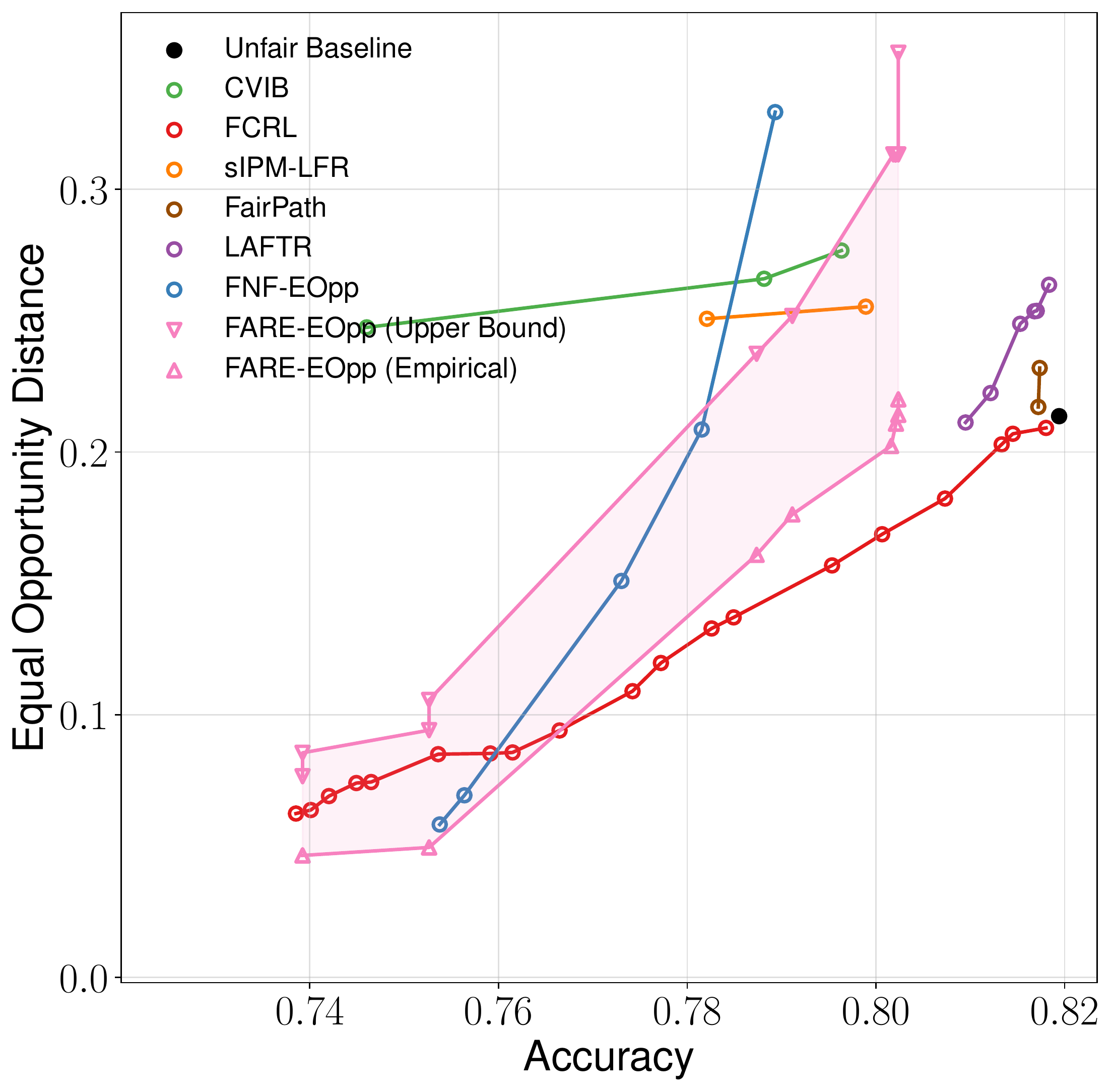}
    \end{subfigure}
    \hfill
    \hspace{1.8em}
    \begin{subfigure}{\relSubfigWidth}
      \centering
      \includegraphics[width=\innerWidth]{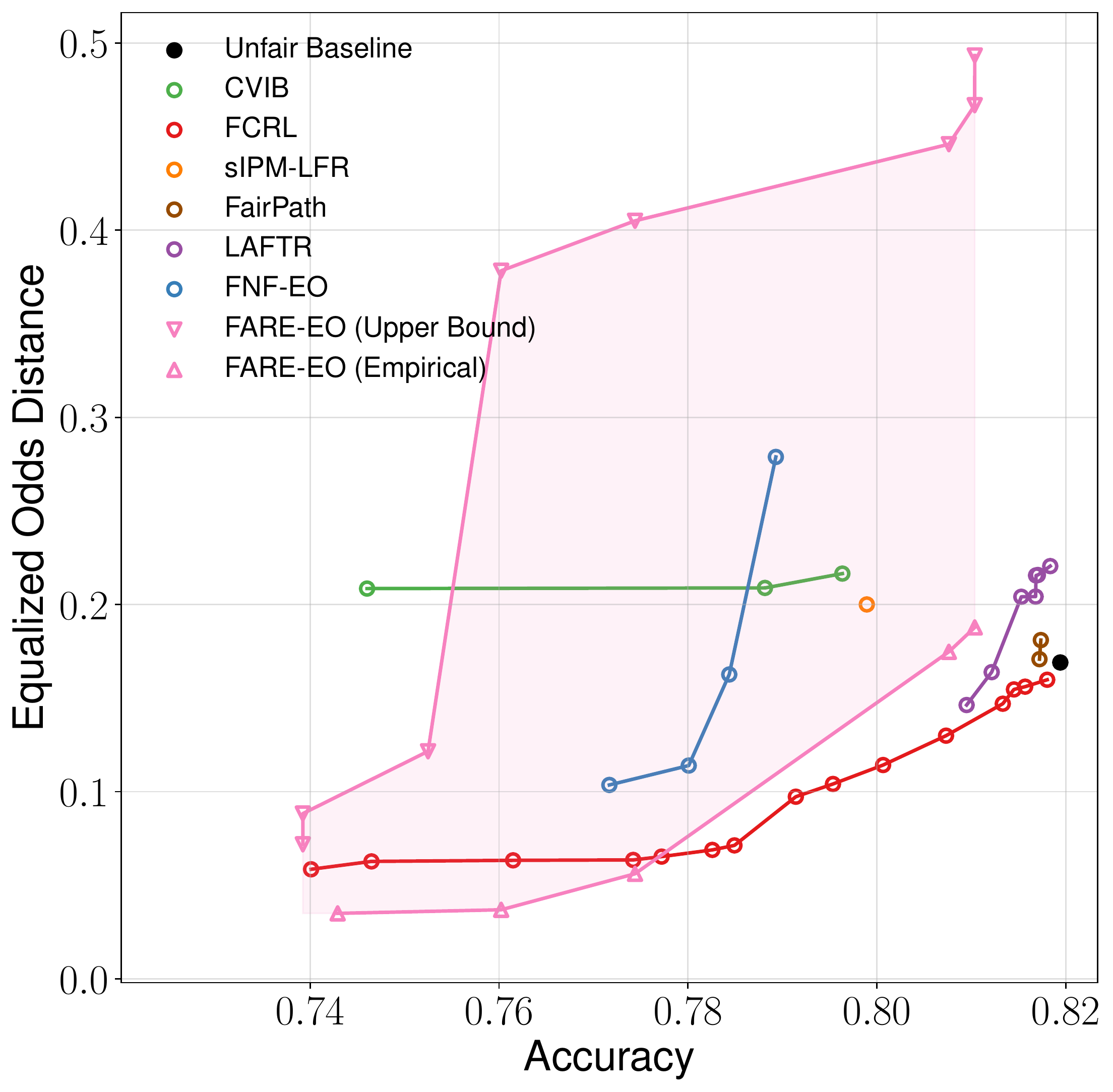}
    \end{subfigure}
    \hspace{1.5em}
    \hfill
    \caption{Generalizing FARE to other metrics: equal opportunity (left) and equalized odds (right), on the Health dataset.}
    \label{fig:newmetrics}
\end{figure*} 
While we presented a straightforward extension, we believe future work on these settings may be able to further improve the tradeoffs and proof tightness by \eg hyperparameter tuning or more elaborate extensions of the statistical procedure for the case of multiple sensitive groups.

\subsection{Other fairness metrics} \label{app:ssec:metrics}
We demonstrate that FARE can easily be extended beyond demographic parity distance $\Delta^{DP}_{\zdi_0, \zdi_1}$, the unfairness metric used in our main results, by considering the commonly used \emph{equal opportunity distance} 
\begin{equation*} 
\Delta^{EOpp}_{\zdi_0, \zdi_1}(g) := \left|\E_{\vz \sim \zdi_0^1}[g(\vz)] - \E_{\vz \sim \zdi_1^1}[g(\vz)]\right|,
    \nonumber
\end{equation*}
and the \emph{equalized odds distance}
\begin{equation*} 
    \Delta^{EO}_{\zdi_0, \zdi_1}(g) := \frac{1}{2} \sum_{y=0}^1 \left|\E_{\vz \sim \zdi_0^y}[g(\vz)] - \E_{\vz \sim \zdi_1^y}[g(\vz)]\right|,
        \nonumber
\end{equation*}
where we use $\zdi_S^Y$ to denote the conditional distribution of $\vz$ where $s=S$ and $y=Y$.

First, we need to extend our statistical procedure from~\cref{sec:proof} to compute certificates of other metrics.
To this end, we observe that
\begin{equation*}
\Delta^{EOpp}_{\zdi_0, \zdi_1} = \Delta^{DP}_{\zdi_0^1, \zdi_1^1},
\end{equation*}
and similarly 
\begin{equation*}
\Delta^{EO}_{\zdi_0, \zdi_1} = \frac{1}{2} \sum_{y=0}^1 \Delta^{DP}_{\zdi_0^y, \zdi_1^y}.
\end{equation*}
This implies that the identical statistical procedure as used for the DP distance can be used to bound EOpp/EO distances, with the only change being the distribution it operates on, \ie we restrict the data we provide to the statistical procedure to only samples where $y=Y$.
Note that for EO we apply the procedure twice with $\epsilon/2$ to retain the total confidence of $1-\epsilon$, per union bound. 

Next, we need to adapt the heuristics used in our tree-based instantiation, to optimize EOpp/EO instead of DP. 
To do this, we simply apply $Gini_s(D)$ in the definition of $FairGini(D)$ only to samples with $y=1$ (for EOpp) and independently to samples with $y=0$ and $y=1$ (for EO), where we take a weighted sum based on the number of samples with each $y$.

\paragraph{Results}
We perform two experiments on the Health dataset (for EO and EOpp) with extensions listed above. 
For baselines, we run FNF with the EO/EOpp argument set, and use the same runs of other baselines as in~\cref{fig:MainHealth}, as these methods do not provide code for metrics other than DP and generally leave such extensions to future work (which originally prompted our focus on DP as the most broadly considered metric).

The results are presented in \cref{fig:newmetrics}.
In both cases, we can come to similar conclusions as in the corresponding DP experiment in~\cref{fig:MainHealth}. 
FARE obtains favorable empirical tradeoffs, while being the only one with an unfairness certificate. 
The certificates are relatively tight, but progressively get looser for higher accuracies---this is much more pronounced for equalized odds. 

%% file: src/app_expdetails.tex
\section{Details of Experimental Evaluation} \label{app:expdetails}

In this section we provide details of our experimental evaluation omitted from the main text.

\begin{table}[t]
  \centering
  \begin{tabular}{lllll}
    \toprule 
    Dataset & $n_{\text{train}}$ & $n_{\text{test}}$ & $R_s$ & $R_y$ \\
    \midrule
    ACSIncome-CA & 165 546 & 18 395 & 0.46 & 0.64 \\
    ACSIncome-US & 1 429 070 & 158 786 & 0.48 & 0.68 \\
    Health & 174 732 & 43 683 & 0.35 & 0.68 \\
    \bottomrule
  \end{tabular}
  \caption{Statistics of evaluated datasets.}
  \label{tab:datasets}
\end{table}

\paragraph{Datasets}

As mentioned in~\cref{sec:experiments}, we perform our experiments on ACSIncome~\citep{ding2021retiring} and Health~\citep{health} datasets.
In~\cref{tab:datasets} we show some general statistics about the datasets: size of the training and test set ($n_{\text{train}}$ and $n_{\text{test}}$), base rate for the sensitive attribute $s$ ($R_s$, percentage of the majority group out of the total population), and base rate for the label $y$ ($R_y$, accuracy of the majority class predictor).

ACSIncome is a dataset recently proposed by~\citet{ding2021retiring} as an improved version of UCI Adult, with comprehensive data from US Census collected across all states and several years (we use 2014). 
The task is to predict whether an individual's income is above \$50,000, and we consider sex as a sensitive attribute.
We evaluate our method on two variants of the dataset: ACSIncome-CA, which contains only data from California, and ACSIncome-US, which merges data from all states and is thus significantly larger but also more difficult, due to distribution shift.
10\% of the total dataset is used as the test set.
We also use the Health dataset~\citep{health}, where the goal is to predict the Charlson Comorbidity Index, and we consider age as a sensitive attribute (binarized by thresholding at 60 years).
For this dataset perform the same preprocessing as~\citet{FNF}, and use 20\% of the total dataset as the test set.

\paragraph{Evaluation procedure} 
For our main experiments, as a downstream classifier we use a 1-hidden-layer neural network with hidden layer size $50$, trained until convergence on representations normalized such that their mean is approximately $0$ and standard deviation approximately $1$. We train the classifier $5$ times and in our main figures report the average test set accuracy, and the maximal DP distance obtained, following the procedure of \citet{FCRL}.

\begin{figure*}[t]
  \newcommand{\relSubfigWidth}{0.31\textwidth}
  \newcommand{\innerWidth}{\textwidth}
  \centering
  \begin{subfigure}{\relSubfigWidth}
      \centering 
      \includegraphics[width=\innerWidth]{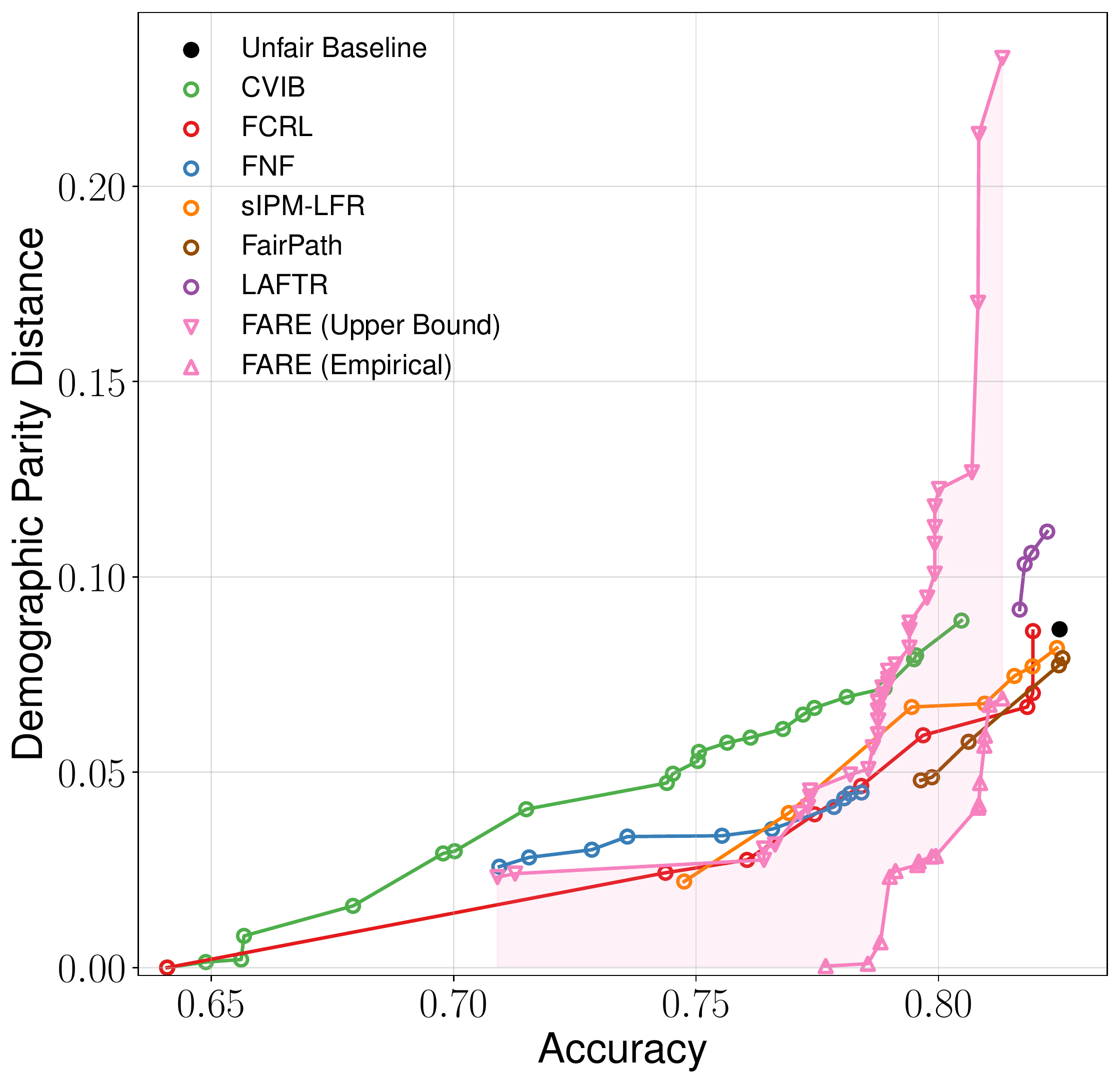}
  \end{subfigure}
  \hfill \hspace{1em}
  \begin{subfigure}{\relSubfigWidth}
    \centering
    \includegraphics[width=\innerWidth]{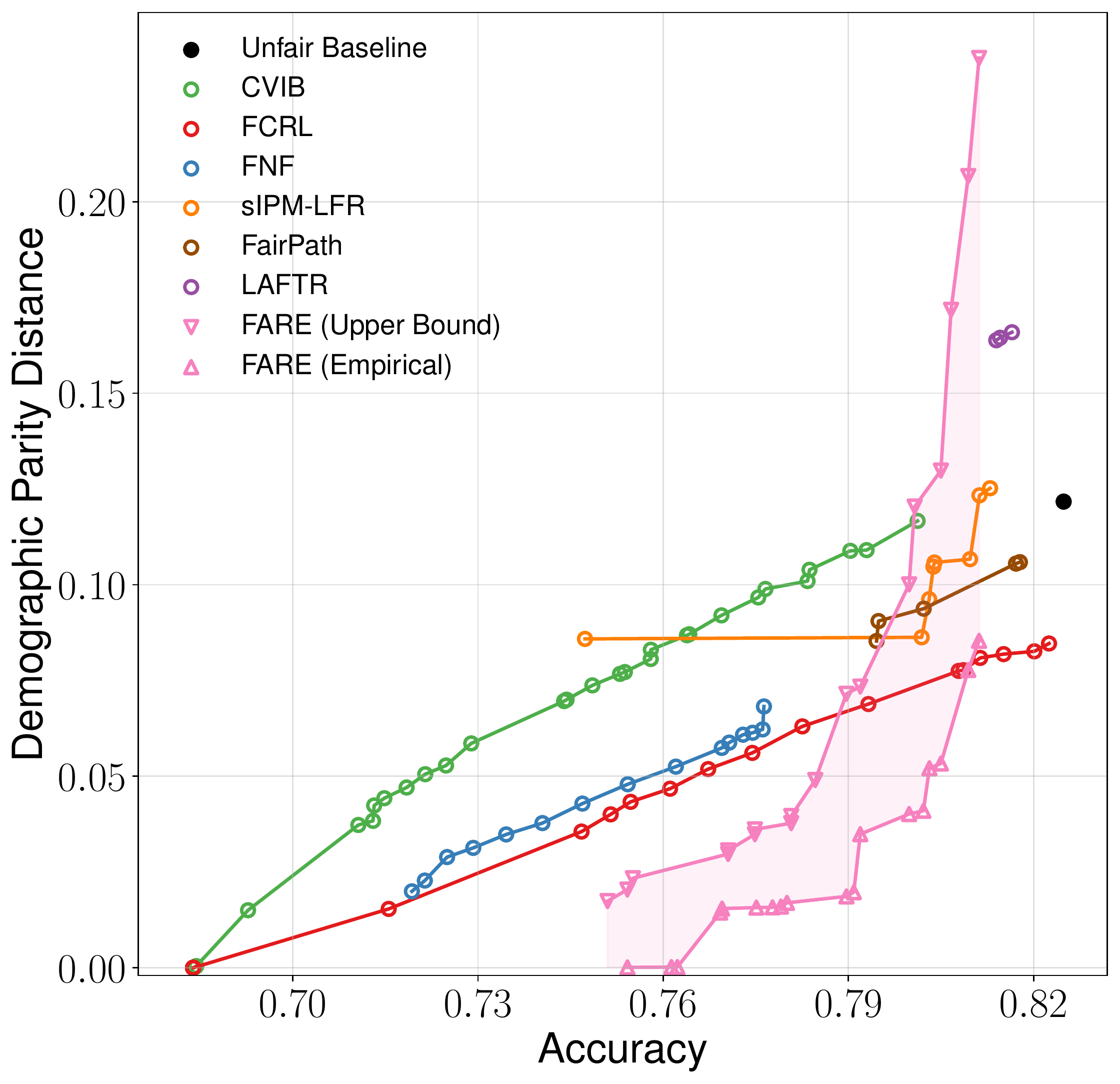}
  \end{subfigure}
  \hfill \hspace{1em}
  \begin{subfigure}{\relSubfigWidth}
    \centering
    \includegraphics[width=\innerWidth]{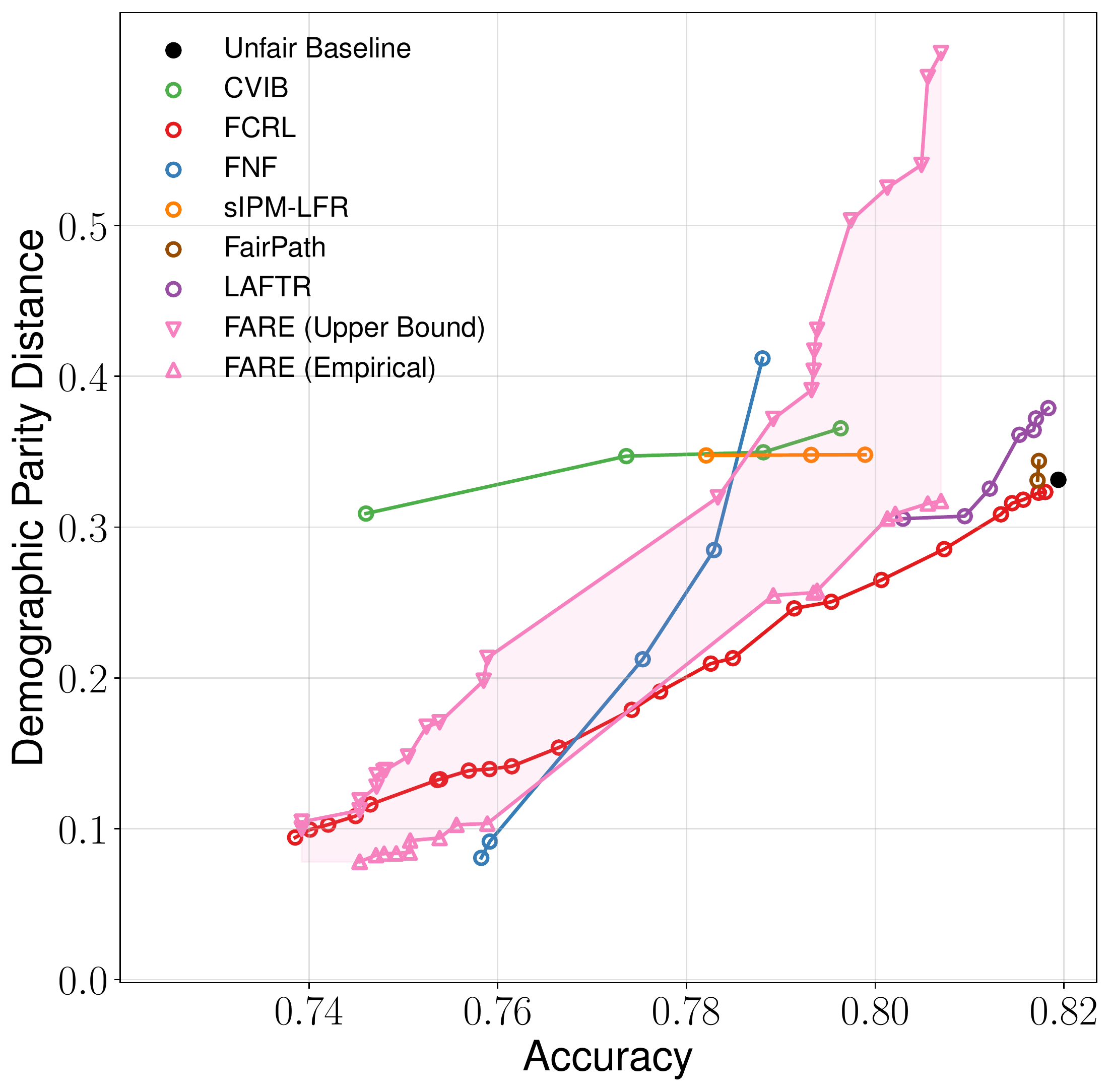}
  \end{subfigure}
  \caption{Extended evaluation of FRL methods on ACSIncome-CA (left), ACSIncome-US (middle) and Health (right).}
  \label{fig:FullMerged}
\end{figure*} 
\paragraph{Hyperparameters} 
For baselines, we follow the instructions in respective writeups, as well as \citet{FCRL} to densely explore an appropriate parameter range for each value (linearly, or exponentially where appropriate), aiming to obtain different points on the accuracy-fairness curve. For CVIB, we explore $\lambda \in [0.01, 1]$ and $\beta \in [0.001, 0.1]$. For FCRL on ACSIncome we explore $\lambda = \beta \in [0.01, 2]$, and for Health  $\lambda \in [0.01, 2]$ and $\beta = 0.5 \lambda$. For FNF, we explore $\gamma \in [0, 1]$. For sIPM, we use $\lambda \in [0.0001, 1.0]$ and $\lambda_F \in [0.0001, 100.0]$, extending the suggested ranges. For FairPath we set the parameter $\kappa \in [0, 100]$. Finally, for LAFTR we use $g \in [0.1, 50]$, extending the range of $[0, 4]$ suggested by \citep{FCRL}. We adjust the parameters for transfer learning whenever supported by the method.

For FARE, there are four hyperparameters: $\gamma$ (used for the criterion, where larger $\gamma$ puts more focus on fairness), $\bar{k}$ (upper bound for the number of leaves), $\underline{n_i}$ (lower bound for the number of examples in a leaf), and $v$ (the ratio of the training set to be used as a validation set). Note that all parameters affect accuracy, empirical fairness, and the tightness of the fairness bound. For example, larger $\underline{n_i}$ is likely to improve the bound by making \cref{lemma2} tighter, as more samples can be used for estimation. For the same reason, increasing $v$ improves the tightness of the bound, but may slightly reduce the accuracy as fewer samples remain in the training set used to train the tree. In our experiments we investigate $\gamma \in [0, 1]$, $\bar{k} \in [2, 200]$, $\underline{n_i} \in [50, 1000]$, $v \in \{0.1, 0.2, 0.3, 0.5\}$. For the upper-bounding procedure, we always set $\epsilon = 0.05$, $\epsilon_b = \epsilon_s = 0.005$, and thus $\epsilon_c = 0.04$. Finally, when sorting categorical features as described in \cref{sec:tree}, we use $q \in \{1, 2, 4\}$ in all cases.

\begin{table}[t]
  \centering
    \resizebox{0.4\textwidth}{!}{ 
    \begin{tabular}{c|c|ccccc}
   \toprule
    $y$ & $\Delta^{DP}_{\zdi_0, \zdi_1}$ & $T^\star$ & FARE & FCRL & FNF & sIPM \\
    \midrule
    \multirow{2}{*}{MIS} & $\leq$ 0.30 & 0.73 & 79.3 & 78.6 & 79.2 & 79.8 \\
     & $\leq$ 0.20 & 0.73 & 79.3 & 78.6 & 78.9 & 79.8 \\
     & $\leq$ 0.15 & 0.73 & 79.3 & 78.6 & 78.9 & 79.6 \\
     & $\leq$ 0.10 & 0.59 & 79.0 & 78.6 & 78.9 & 79.0 \\
     & $\leq$ 0.05 & 0.57 & 78.7 & 78.6 & 78.7 & 78.6 \\
    \midrule
    \multirow{2}{*}{NEU} & $\leq$ 0.30 & 0.72 & 73.2 & 72.4 & 71.9 & 78.8 \\
     & $\leq$ 0.20 & 0.72 & 73.2 & 72.4 & 71.9 & 76.6 \\
     & $\leq$ 0.15 & 0.72 & 73.2 & 72.4 & 71.8 & 73.2 \\
     & $\leq$ 0.10 & 0.42 & 73.1 & 72.2 & 71.8 & / \\
     & $\leq$ 0.05 & 0.43 & 72.1 & 71.4 & 71.7 & / \\
    \midrule
    \multirow{2}{*}{ART} & $\leq$ 0.30 & 0.55 & 74.4 & 70.7 & 68.9 & 78.3 \\
     & $\leq$ 0.20 & 0.55 & 74.4 & 70.7 & 68.9 & 78.3 \\
     & $\leq$ 0.15 & 0.48 & 74.2 & 70.1 & 68.9 & / \\
     & $\leq$ 0.10 & 0.12 & 69.5 & 69.6 & 68.7 & / \\
     & $\leq$ 0.05 & 0.12 & 69.5 & 69.5 & 68.5 & / \\
    \midrule
    \multirow{2}{*}{MET} & $\leq$ 0.30 & 0.48 & 74.0 & 72.5 & 76.2 & / \\
     & $\leq$ 0.20 & 0.48 & 69.8 & 69.2 & 75.0 & / \\
     & $\leq$ 0.15 & 0.33 & 68.7 & 67.9 & 73.2 & / \\
     & $\leq$ 0.10 & 0.12 & 66.1 & 66.7 & 73.2 & / \\
     & $\leq$ 0.05 & 0.12 & 66.1 & 65.3 & / & / \\
    \midrule
    \multirow{2}{*}{MSC} & $\leq$ 0.30 & 0.59 & 71.3 & 70.5 & 73.5 & 77.6 \\
     & $\leq$ 0.20 & 0.48 & 67.4 & 70.5 & 73.0 & / \\
     & $\leq$ 0.15 & 0.12 & 63.0 & 69.7 & / & / \\
     & $\leq$ 0.10 & 0.12 & 63.0 & 69.0 & / & / \\
     & $\leq$ 0.05 & 0.12 & 63.0 & / & / & / \\
    \bottomrule
    \end{tabular}}
    \caption{Extended transfer learning results on Health.}
    \label{tab:transferext}
  \end{table} 
   
\paragraph{Omitted details of additional experiments}
For the experiment in \cref{fig:downstream_main} we explore the following classifiers: (i) 1-hidden-layer neural network (1-NN) with hidden layer sizes $50$ and $200$, (ii) 2-NN with hidden layers of size $(50, 50)$, as well as $(200, 100)$, (iii) logistic regression, (iv) random forest classifier with $100$ and $1000$ estimators, (v) decision tree with $100$ and an unlimited number of leaf nodes. We train all these classifiers with a standardization preprocessing step as described above. We further train one variant of 1-NN, 2-NN, random forest, and logistic regression, on unnormalized data. All described models are trained both to predict the task label $y$, and to maximize unfairness, \ie predict $s$, leading to 24 evaluated models. 

For transfer learning (\cref{tab:transfer}), the five transfer tasks represent prediction of the following attributes from the Health dataset: \textsc{MISCHRT} (MIS), \textsc{NEUMENT} (NEU), \textsc{ARTHSPIN} (ART), \textsc{METAB3} (MET), \textsc{MSC2a3} (MSC).
 

%% file: src/app_extendedresults.tex
\section{Extended Results} \label{app:extendedresults}

We provide the extended results of our main experiments, including two originally excluded methods, LAFTR and FairPath in \cref{fig:FullMerged}, corresponding to \cref{fig:MainACS} and \cref{fig:MainHealth}.

Additionally, in~\cref{tab:transferext} we provide extended results of our transfer learning experiments, showing the accuracy values for thresholds $\Delta^{DP}_{\zdi_0, \zdi_1} \in \{0.30, 0.20, 0.15, 0.10, 0.05\}$. We can observe similar trends as shown in~\cref{tab:transfer}.

%% file: src/app_efficiency.tex
\section{Computational Efficiency Experiments} \label{app:efficiency}

\begin{table*}[t] 
    \centering

	\renewcommand{\arraystretch}{1.2}
    \newcommand{\threecol}[1]{\multicolumn{3}{c}{#1}}

    \resizebox{0.9\textwidth}{!}{
    \begin{tabular}{@{}rrrrrrrrrr@{}}
    \toprule
    & \threecol{Time} & \threecol{RAM} & \threecol{VRAM} \\ 
	& M=1 & M=4 & M=16 & M=1 & M=4 & M=16 & M=1 & M=4 & M=16 \\
      \midrule
      \text{FARE}  & \textbf{3 sec} & \textbf{11 sec} & \textbf{55 sec} & \text{2.6 GB} & \text{3.0 GB} & \text{9.1 GB} & \textbf{0.0 GB} & \textbf{0.0 GB} & \textbf{0.0 GB} \\
    \text{FNF}  & \text{19 min} & \text{1 h 10 min} & \text{4 h 15 min} & \textbf{2.3 GB} & \text{4.3 GB} & \text{11.1 GB} & \text{1.6 GB} & \text{2.0 GB} & \text{3.2 GB}  \\
    \text{FCRL} & \text{57 min} & \text{3 h 51 min} & \text{15 h 38 min} & \text{2.4 GB} & \textbf{2.4 GB} & \textbf{2.5 GB} & \text{1.2 GB} & \text{2.7 GB} & \text{8.7 GB}  \\
    \text{CVIB} & \text{42 min} & \text{2 h 54 min} & \text{11 h 30 min} & \text{2.4 GB} & \textbf{2.4 GB} & \textbf{2.5 GB} & \text{1.2 GB} & \text{2.7 GB} & \text{8.7 GB}  \\
    \text{sIPM-LFR}& \text{23 min} & \text{1 h 35 min} & \text{OOM} & \text{3.1 GB} & \text{7.4 GB} & \text{$>$24.5 GB} & \text{2.0 GB} & \text{6.4 GB} & \text{$>$12 GB}  \\
    \bottomrule
    \end{tabular}}
    \caption{A study of computational efficiency of FRL methods. }
    \label{table:efficiency}
\end{table*} 

We investigate the computational efficiency and scalability of FARE in comparison to the baselines.
We perform our measurements using the ACSIncome-CA dataset repeated $M$ times (to observe the effects of data size, as in~\cref{fig:bigdata}). 
For each method, we use a single set of parameters (\ie a single point from \cref{fig:MainACS}, left). 
We use a \emph{single core} of the i9-7900X CPU Intel CPU that has clock speed of 3.30GHz. 
All methods were given a single NVIDIA 1080 Ti GPU with 12 GB of VRAM, except FARE which does not require a GPU. 
We report CPU RAM, GPU RAM (VRAM) and runtimes for all methods for $M \in \{1, 4, 16\}$. 

The results are shown in~\cref{table:efficiency}.
As we noted in the main paper, FARE takes seconds to execute, which is orders of magnitude faster compared to all other methods, while also not requiring GPU support (\ie we have VRAM requirements of $0.0$ GB in~\cref{table:efficiency}). 
Further, FARE’s runtime scales at most linearly with the dataset size, and its RAM usage scales well with $M$, and can easily be scaled to very large $M$ on moderate hardware (as few 100s of GBs of CPU RAM are common on modern systems). 
In contrast, all other methods take hours for large $M$ even with GPUs. 
FCRL, CVIB and sIPM-LFR use significant GPU RAM (scaling poorly with $M$). 
While FNF’s GPU memory usage comparatively scales well with $M$, it still uses more CPU RAM than FARE.
 

%% file: src/app_moreexperiments.tex
\section{Additional Experiments} \label{app:moreexperiments}

In this section, we provide an ablation study of $\bar{k}$ and additional experiments investigating the robustness of FARE to different downstream classifiers, dataset size, distribution shifts, missing data, and sensitive attribute imbalance.

\subsection{Experiments with different downstream classifiers} \label{app:moreexp:downstream}
We compare different FRL methods on ACSIncome-CA, similarly to \cref{fig:MainACS}~(left), in the case of different downstream classifiers. In \cref{fig:cadown}, we show the results on 4 additional downstream classifiers:
\begin{itemize}
	\item{Decision Tree with maximum $2500$ leaves }
	\item{Random Forest using $100$ trees }
	\item{Logistic Regression}
	\item{Two-layer neural network with $50$ neurons per layer}
\end{itemize} 

We observe that the general trends observed in \cref{fig:MainACS}~(left) for the different FRL methods hold regardless of the downstream classifier choice. We also see that the gap of our method to the maximum achievable accuracy is the smallest when the downstream classifier is a tree. This is unsurprising given that FARE's own representations are based on trees. Further, we see that the complex feature extraction of FCRL and sIPM enables higher accuracy than the unfair baseline in case of a simple classifier such as a decision tree.

\subsection{Ablation study} \label{app:moreexp:ablation}
We present an ablation study of $\bar{k}$ in \cref{table:ablation}.
We explore 3 different settings of FARE: (i) \emph{Fair}, with $\gamma=0.999, \underline{n_i}=1000, v=0.5$, \emph{Balanced}, with $\gamma=0.85, \underline{n_i}=100, v=0.3$, and \emph{Accurate}, with $\gamma=0.3, \underline{n_i}=10, v=0.1$. 
In each setting we do a run with each $\bar{k} \in \{3,5,8,20,50\}$, and measure the accuracy, DP distance (unfairness), and $T^\star$, the DP distance certificate.
We can observe that in all three settings, as expected, increasing $k$ generally improves accuracy, but makes the certificate higher (and looser).
 
\begin{table*}[t] 
    \centering

	\renewcommand{\arraystretch}{1.2}
    \newcommand{\threecol}[1]{\multicolumn{3}{c}{#1}}

    \resizebox{0.7\textwidth}{!}{
    \begin{tabular}{@{}rccccccccc@{}}
    \toprule
    & \threecol{\emph{Fair}} & \threecol{\emph{Balanced}} & \threecol{\emph{Accurate}} \\ 
	$\bar{k}$ & Acc. & DP Dist. & $T^\star$ & Acc. & DP Dist. & $T^\star$ &  Acc. & DP Dist. & $T^\star$ \\
      \midrule
	  $3$ & 0.735 & \textbf{0.001} & \textbf{0.028} & 0.786 & \textbf{0.029} & \textbf{0.067} & 0.786 & 0.081 & \textbf{0.138} \\
	  $5$ & 0.760 & \textbf{0.001} & 0.035 & 0.800 & 0.057 & 0.105 & 0.799 & 0.074 & 0.150 \\
	  $8$ & 0.773 & 0.007 & 0.041 & 0.800 & 0.057 & 0.113 & 0.802 & 0.080 & 0.163 \\
	  $20$ & 0.774 & 0.008 & 0.072 & 0.803 & 0.045 & 0.144 & 0.810 & 0.062 & 0.229 \\
	  $50$ & \textbf{0.777} & 0.004 & 0.082 & \textbf{0.805} & 0.049 & 0.191 & \textbf{0.812} & \textbf{0.058} & 0.303 \\
    \bottomrule
    \end{tabular}}
    \caption{An ablation study of the FARE hyperparameter $\bar{k}$.}
    \label{table:ablation}
 \end{table*} 

 \begin{figure*}[t]
    \newcommand{\relSubfigWidth}{0.38\textwidth}
    \newcommand{\innerWidth}{\textwidth}
    \centering
    \hfill
    \begin{subfigure}{\relSubfigWidth}
        \centering 
        \includegraphics[width=\innerWidth]{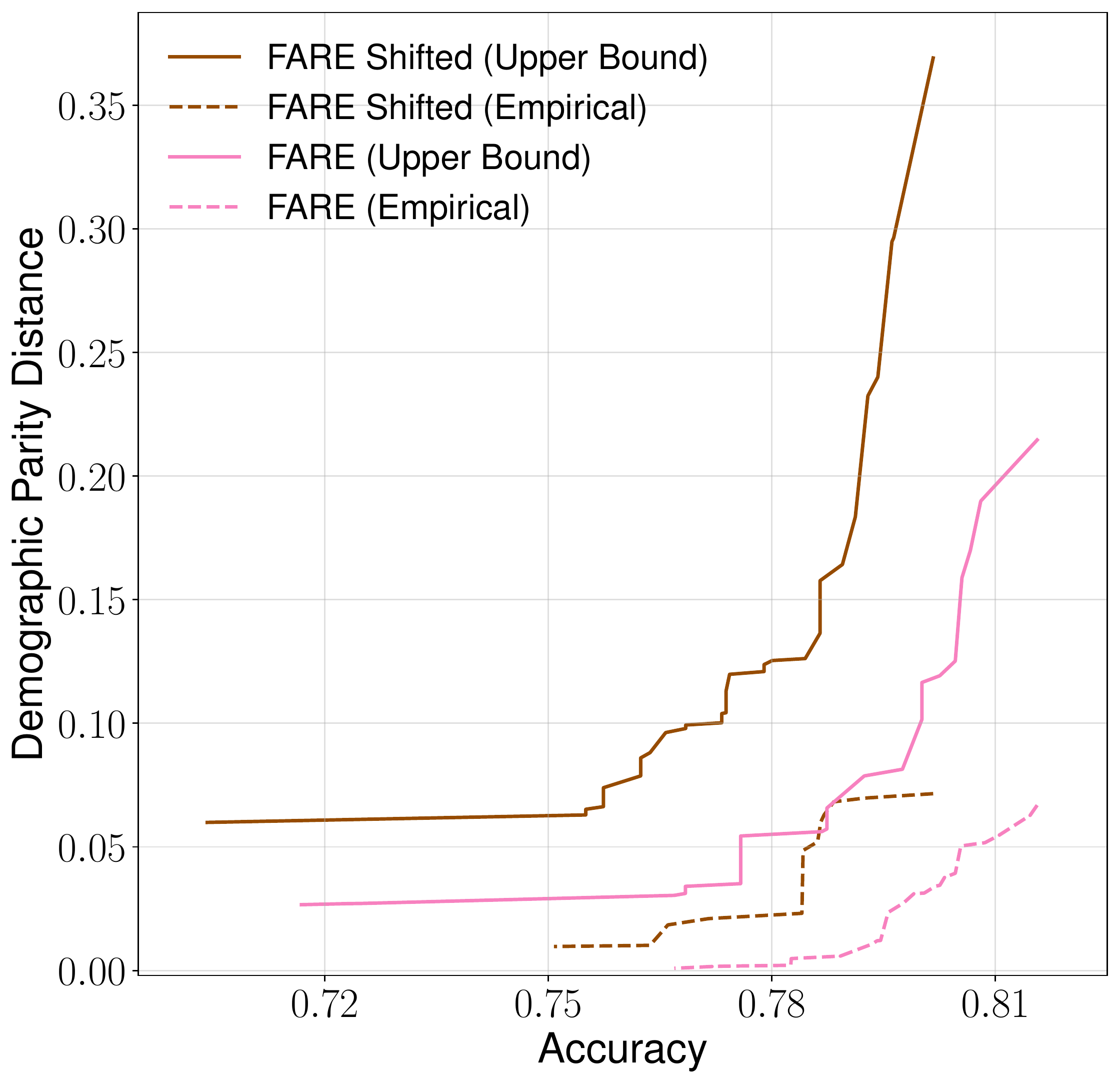}
    \end{subfigure}
    \hfill
    \hspace{1.8em}
    \begin{subfigure}{\relSubfigWidth}
      \centering
      \includegraphics[width=\innerWidth]{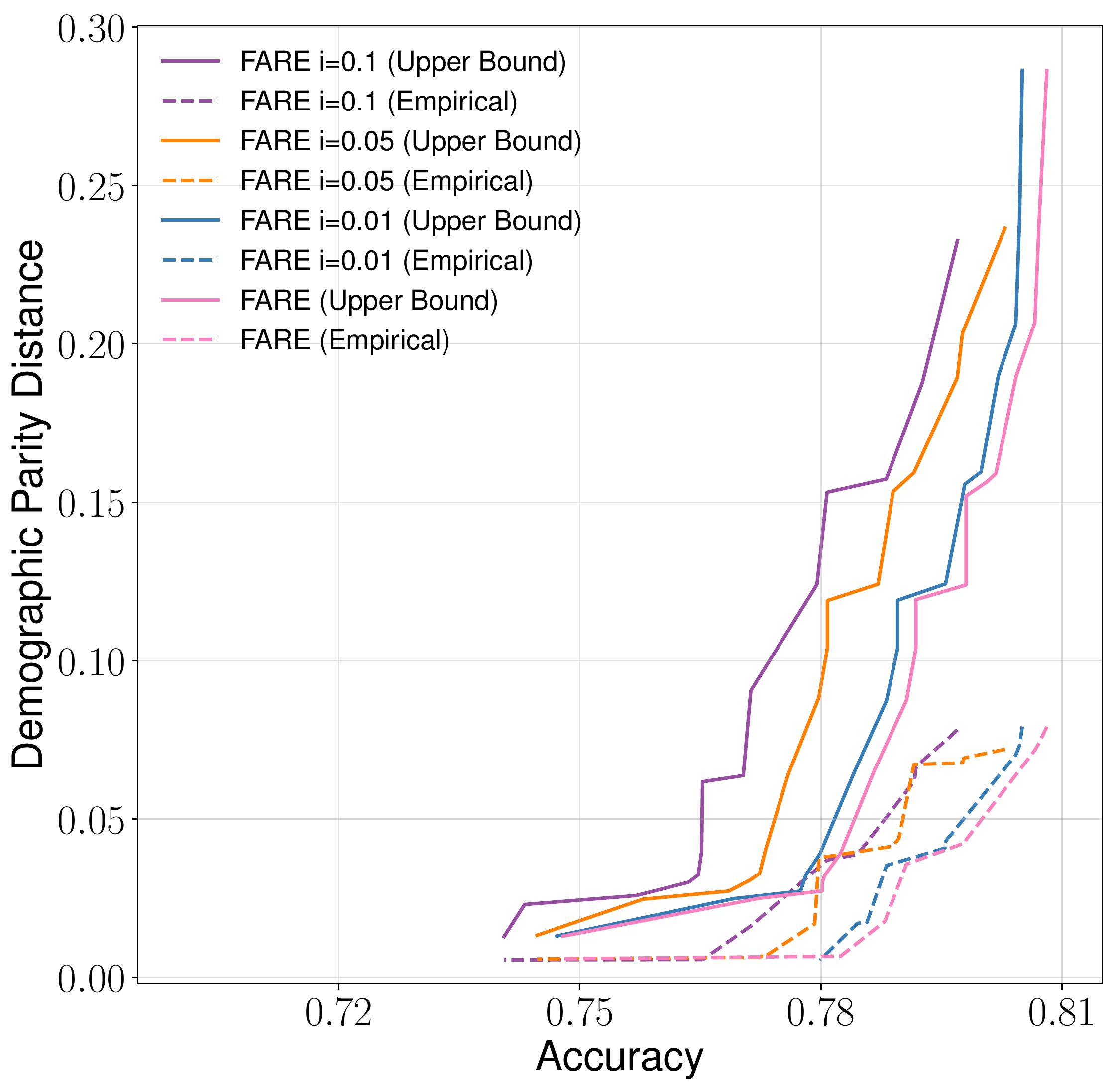}
    \end{subfigure}
    \hspace{1.5em}
    \hfill
    \caption{Exploring the robustness of FARE to distribution shifts and missing values.}
    \label{fig:shiftimpute}
\end{figure*}  
 
\begin{figure}[t]
	\newcommand{\relSubfigWidth}{0.28\textwidth}
	\newcommand{\innerWidth}{\textwidth}
	\centering
	\begin{subfigure}[b]{\relSubfigWidth}
		\centering 
		\includegraphics[width=\innerWidth]{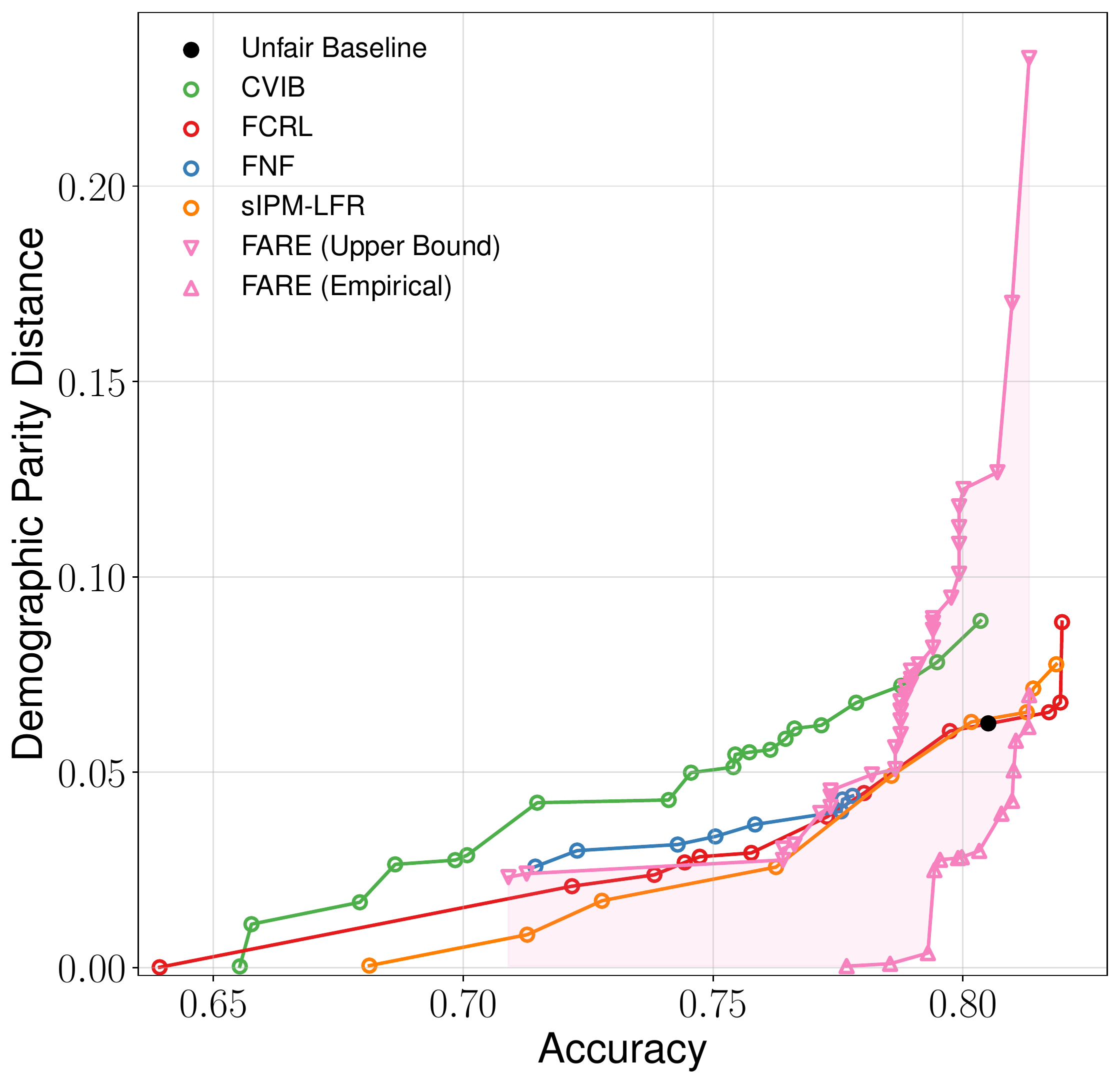}
		\caption{Decision Tree}
		\vspace{0.3em}
	\end{subfigure}
	\begin{subfigure}[b]{\relSubfigWidth}
		\centering
		\includegraphics[width=\innerWidth]{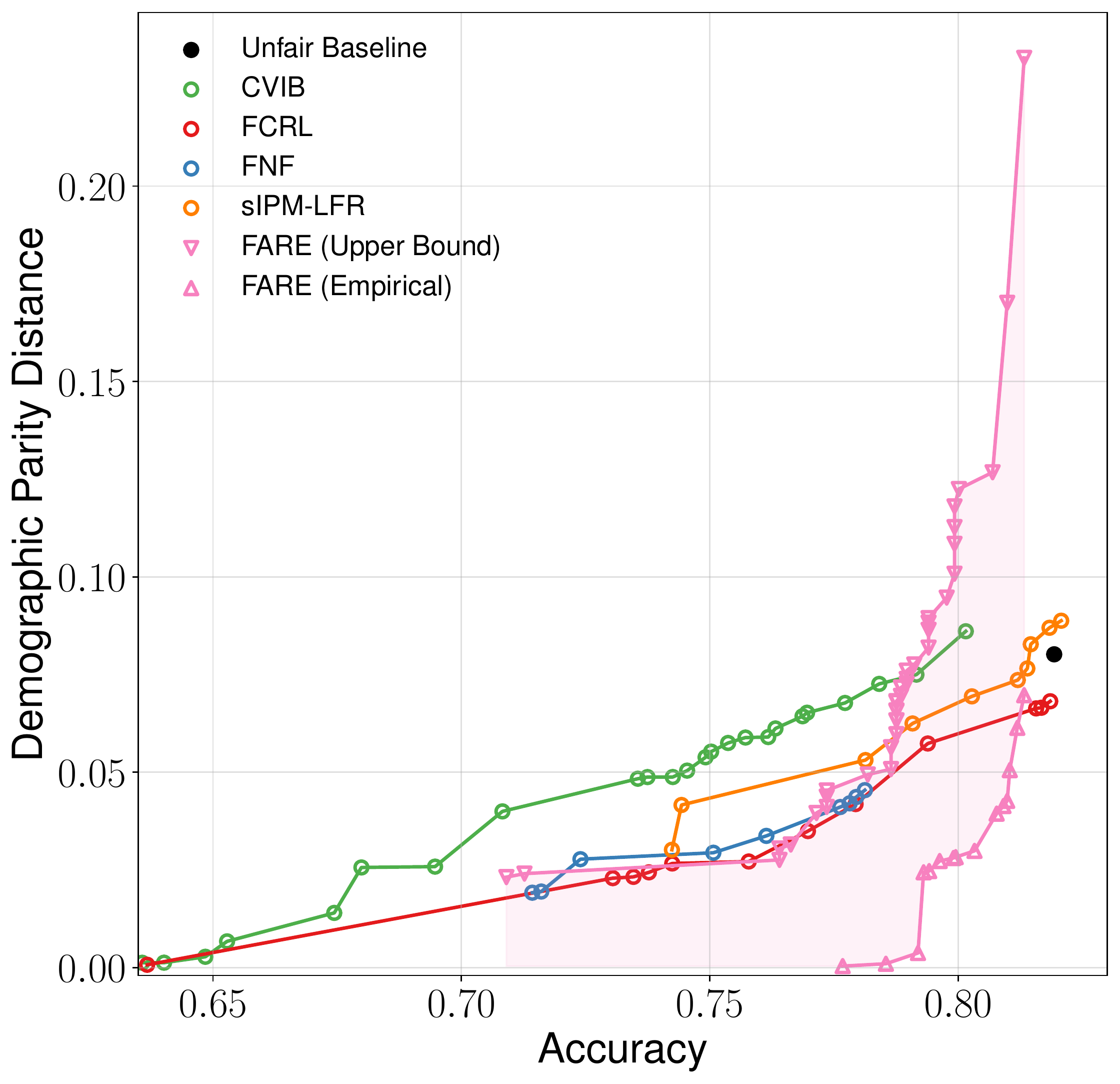}
		\caption{Random Forest}
		\vspace{0.3em}
	\end{subfigure}
	\begin{subfigure}[b]{\relSubfigWidth}
		\centering 
		\includegraphics[width=\innerWidth]{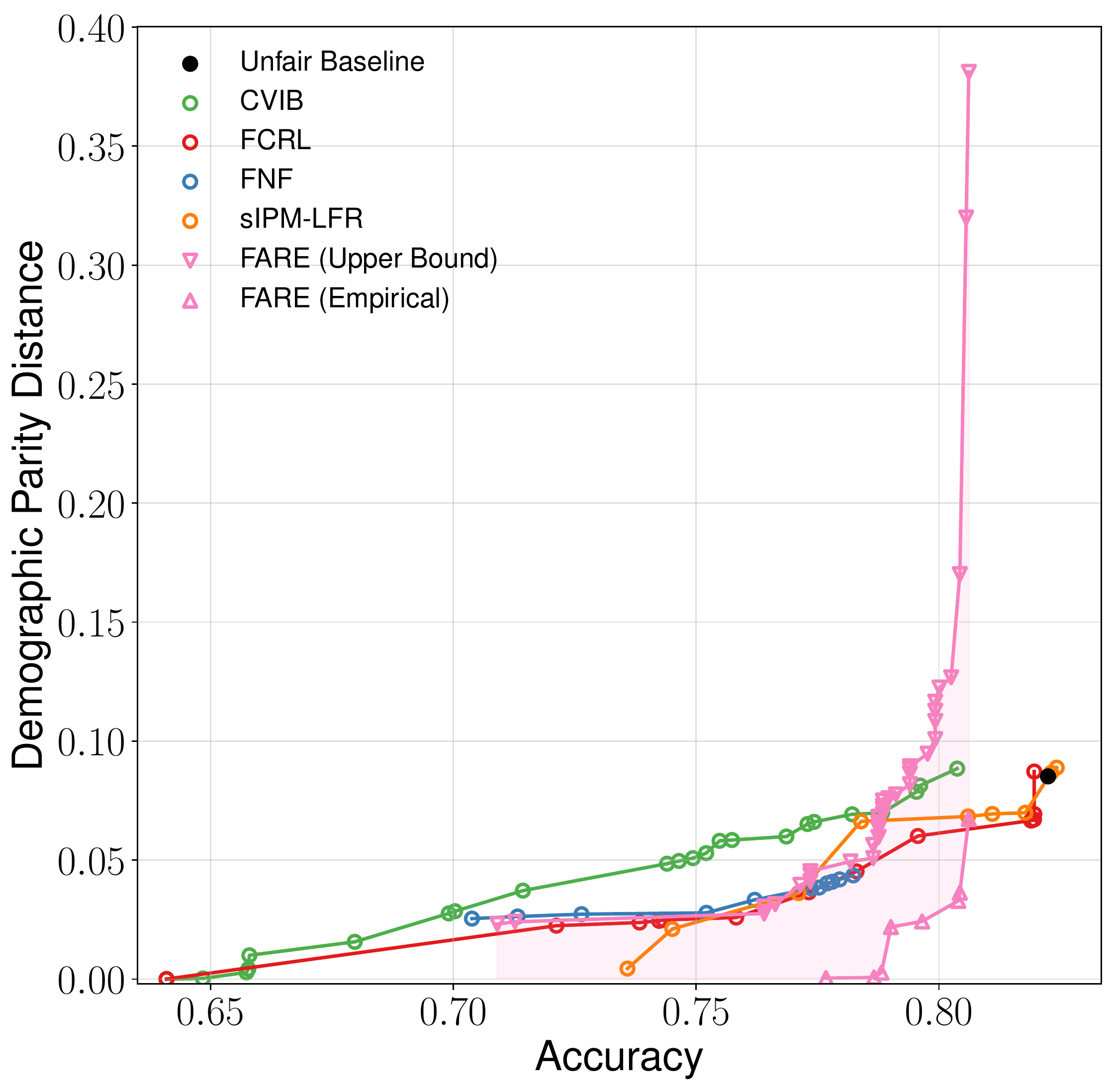}
		\caption{Logistic Regression}
		\vspace{0.3em}
	\end{subfigure}
	\begin{subfigure}[b]{\relSubfigWidth}
		\centering
		\includegraphics[width=\innerWidth]{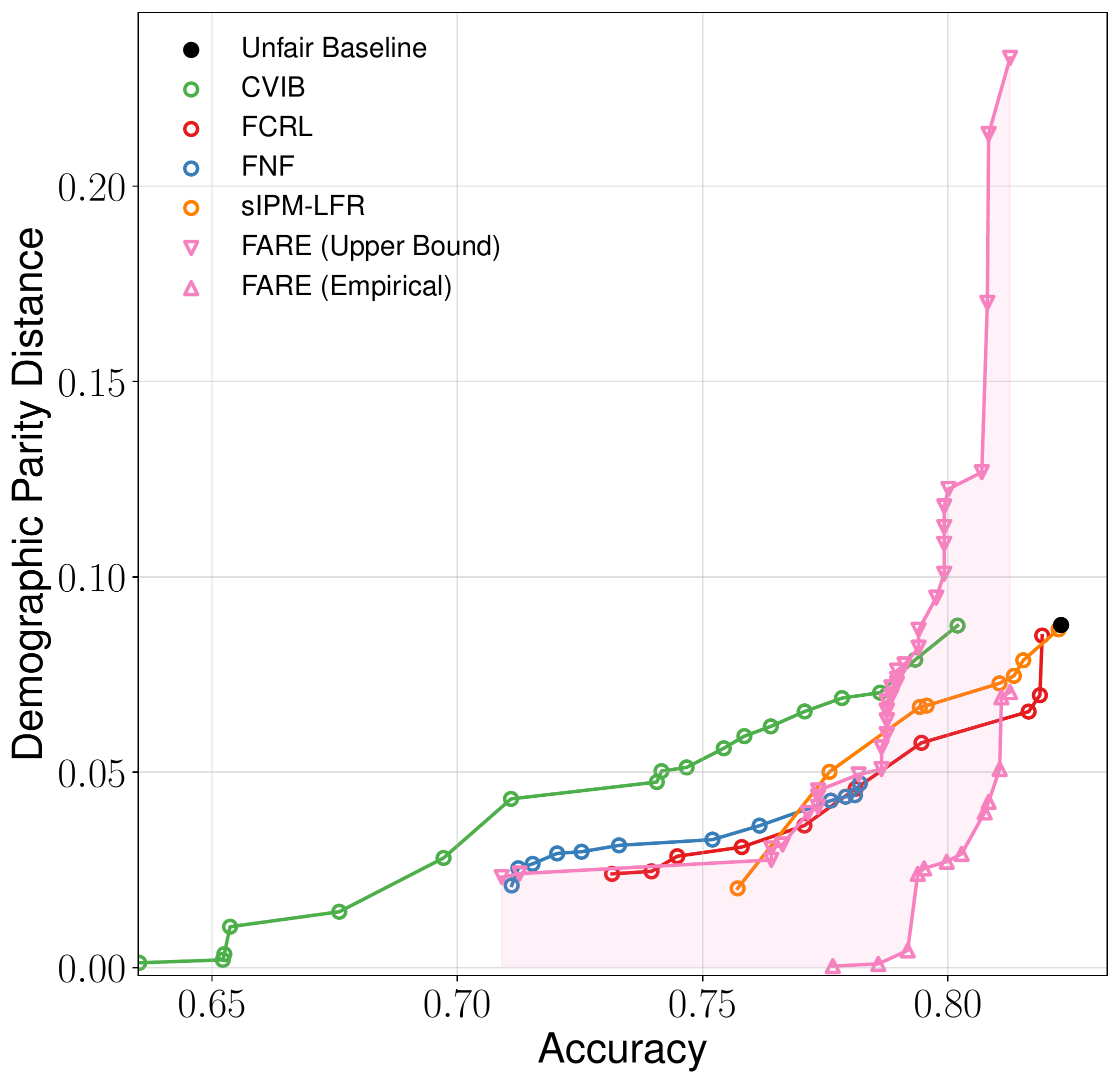}
		\caption{Two-Layer NN}
	\end{subfigure}
	\caption{Comparison between different downstream classifiers on different FRL methods on ACSIncome-CA.}
	\label{fig:cadown}
\end{figure}

\subsection{Performance gap on larger datasets} \label{app:moreexp:larger}
\begin{figure}[t]
	\centering
	\includegraphics[width=0.9\linewidth]{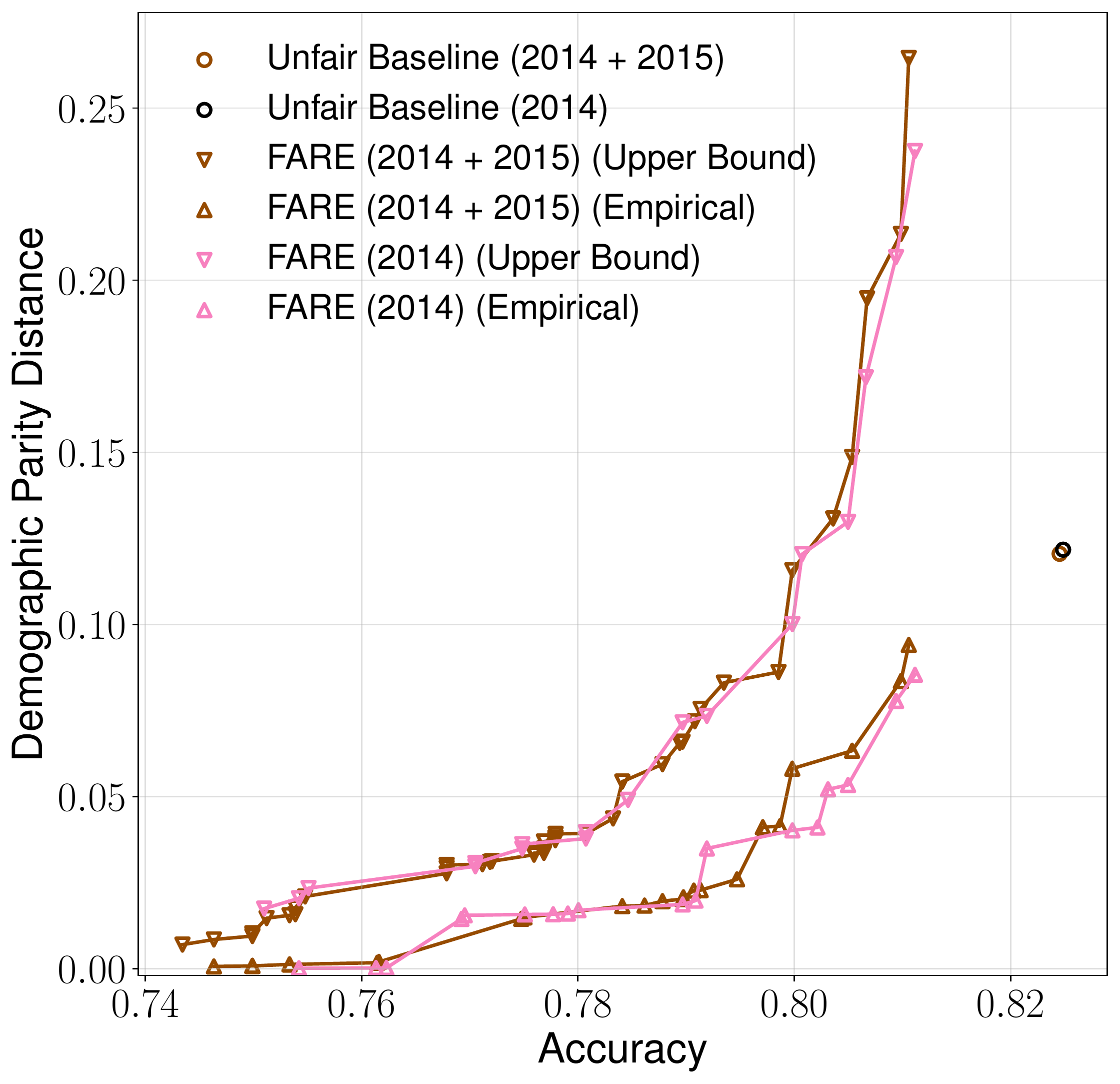}
	\caption{Comparison of the performance gap between FARE and the Unfair Baseline on ACSIncome-US for a single year and two years.}
	\label{fig:2years}
\end{figure}

Next, we explore whether the small performance gap we observed in our main results, between FARE's most accurate model and the unfair baseline, widens for larger datasets. To this end, we merge two ACSIncome-US datasets from two consecutive years (2014 and 2015) and compare the results to the single year dataset from 2014, shown in \cref{fig:MainACS}~(right). We note that the merged dataset has roughly 2x the number of data points.
The comparison between the merged and single-year datasets is shown in \cref{fig:2years}. We observe almost no difference between the results on the two datasets for the unfair baseline as well as the empirical and provable fairness of our method. This suggests that the complexity of the dataset is a more important factor than the data volume for the observed performance gap.

\subsection{Distribution shifts and imputation} \label{app:moreexp:shiftimpute}

Next, we briefly consider two aspects that were not our main focus---the robustness to non-IID (distribution shift) and missing data.
To measure distribution shift, we train FARE on ACSIncome-CA data from 2015 and compare its results when evaluated on test data from 2015 (the standard case) and 2016 (the \emph{FARE Shifted} case).
Regarding missing data, we use ACSIncome-US (as in \cref{fig:MainACS}, right), and split its data in two parts, training FARE on the first part, and evaluating it (\ie training downstream classifiers and computing the certificate) on the other part, where we randomly remove the fraction $i$ of feature values, and impute them (using the most common value for categorical, and mean for continuous features). We use values $i \in \{0, 0.01, 0.05, 0.1\}$.

The results are shown in~\cref{fig:shiftimpute}. 
In both cases, as expected, we see some degradation of results---such concerns could be studied more in follow-up work.

\subsection{Effect of sensitive attribute imbalance} \label{app:moreexp:sens}
\begin{figure}[t]
	\centering
	\includegraphics[width=0.9 \linewidth]{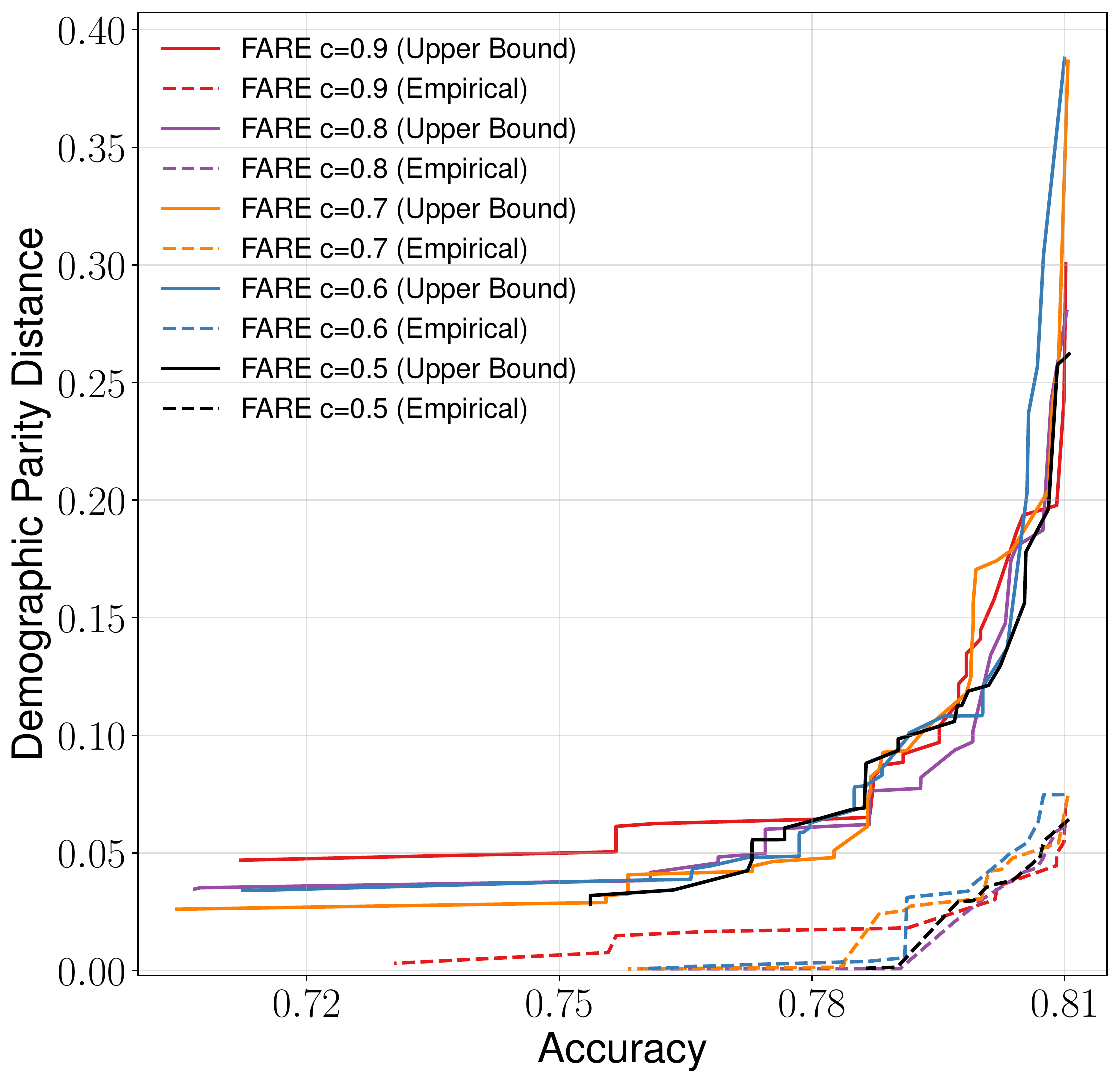}
	\caption{Evaluation of FARE at different levels of imbalance in the sensitive attribute (denoted by $c$) on randomly sampled subsets of ACSIncome-CA of the same size.}
	\label{fig:unbalance}
\end{figure}
Finally, we study the effect of imbalance in the sensitive attribute on the fairness and accuracy of FARE.
Let $c$ denote the level of imbalance of each training set (i.e., the number of data points in the more common sensitive class divided by the total number of data points in the set).
For each value of $c$ we are interested in, we sample a random subset of size 49 053 from the original ACSIncome-CA training dataset (out of 165 546 data points in total), ensuring that the level of imbalance is exactly $c$.
We use this number of samples, as this is the largest number for which we can have the same dataset size for each $c$, ensuring the fair comparison.

We train FARE on each subset separately and show Pareto plots, similar to those in \cref{fig:MainACS} and \cref{fig:MainHealth}, in \cref{fig:unbalance}.  We observe that FARE is robust to imbalance, as even for $c=0.9$, we only see a small difference in our Pareto curves (and only in the low-accuracy regime).